\documentclass[12pt,twoside]{report}
\usepackage[utf8]{inputenc}
\usepackage{graphicx}
\usepackage[a4paper,width=150mm,top=25mm,bottom=25mm,bindingoffset=6mm]{geometry}
\graphicspath{ {images/} }

\usepackage{amsthm, amsmath, amssymb, bm}
\usepackage{algorithm, algorithmic}
\usepackage{url}
\usepackage{hyperref}
\usepackage{graphicx}

\usepackage[round]{natbib}

\numberwithin{equation}{section}
\theoremstyle{definition} \newtheorem{idealmmgan}{Definition}[section]
\theoremstyle{definition} \newtheorem{idealremark}[idealmmgan]{Remark}
\theoremstyle{definition} \newtheorem{idealremark2}[idealmmgan]{Remark}
\theoremstyle{plain} \newtheorem{optdim}[idealmmgan]{Theorem}
\theoremstyle{plain} \newtheorem{optdimcor}[idealmmgan]{Corollary}
\theoremstyle{definition} \newtheorem{kljsdef}[idealmmgan]{Definition}
\theoremstyle{plain} \newtheorem{kljsprop}[idealmmgan]{Proposition}
\theoremstyle{plain} \newtheorem{kljscor}[idealmmgan]{Corollary}
\theoremstyle{plain} \newtheorem{immsol}[idealmmgan]{Theorem}
\theoremstyle{plain} \newtheorem{immconv}[idealmmgan]{Theorem}
\theoremstyle{definition} \newtheorem{mmgan}[idealmmgan]{Definition}
\theoremstyle{plain} \newtheorem{optgen}[idealmmgan]{Proposition}
\theoremstyle{plain} \newtheorem{optgencor}[idealmmgan]{Corollary}

\theoremstyle{definition} \newtheorem{advdiv}{Definition}[section]
\theoremstyle{definition} \newtheorem{immadvdiv}[advdiv]{Example}
\theoremstyle{definition} \newtheorem{ipmadvdiv}[advdiv]{Example}
\theoremstyle{plain} \newtheorem{ipmprop}[advdiv]{Proposition}
\theoremstyle{definition} \newtheorem{w1}[advdiv]{Definition}
\theoremstyle{plain} \newtheorem{krduality}[advdiv]{Theorem}
\theoremstyle{plain} \newtheorem{kripmcor}[advdiv]{Corollary}
\theoremstyle{plain} \newtheorem{krdistcor}[advdiv]{Corollary}
\theoremstyle{definition} \newtheorem{tvdist}[advdiv]{Definition}
\theoremstyle{plain} \newtheorem{tvipmprop}[advdiv]{Proposition}
\theoremstyle{plain} \newtheorem{tvdistcor}[advdiv]{Corollary}
\theoremstyle{definition} \newtheorem{weaktopology}[advdiv]{Definition}
\theoremstyle{plain} \newtheorem{simplefailure}[advdiv]{Example}
\theoremstyle{plain} \newtheorem{wcts}[advdiv]{Theorem}
\theoremstyle{plain} \newtheorem{minemdist}[advdiv]{Corollary}
\theoremstyle{plain} \newtheorem{strengths}[advdiv]{Theorem}
\theoremstyle{plain} \newtheorem{princproc}[advdiv]{Theorem}
\theoremstyle{definition} 
\theoremstyle{plain} 
\theoremstyle{plain} 
\theoremstyle{plain} 
\theoremstyle{plain} 
\theoremstyle{plain} 
\theoremstyle{plain} 

\theoremstyle{definition} \newtheorem{nsgan}{Definition}[section]
\theoremstyle{definition} \newtheorem{game}[nsgan]{Definition}
\theoremstyle{definition} \newtheorem{nedef}[nsgan]{Definition}
\theoremstyle{definition} \newtheorem{zerosum}[nsgan]{Definition}
\theoremstyle{definition} \newtheorem{mixgame}[nsgan]{Definition}
\theoremstyle{definition} \newtheorem{mixnash}[nsgan]{Definition}
\theoremstyle{plain} 
\theoremstyle{definition} \newtheorem{fanmmrem}[nsgan]{Remark}
\theoremstyle{plain} \newtheorem{glicksberg}[nsgan]{Theorem}
\theoremstyle{definition} \newtheorem{nndist}[nsgan]{Definition}
\theoremstyle{definition} \newtheorem{nnrem}[nsgan]{Remark}
\theoremstyle{definition} \newtheorem{epsiloneq}[nsgan]{Definition}
\theoremstyle{plain} \newtheorem{arorapure}[nsgan]{Theorem}
\theoremstyle{definition} \newtheorem{cycles}[nsgan]{Example}

\theoremstyle{definition} \newtheorem{ctsgame}{Definition}[section]
\theoremstyle{plain} \newtheorem{glickprop1}[ctsgame]{Lemma}
\theoremstyle{definition} \newtheorem{glickdef1}[ctsgame]{Definition}
\theoremstyle{plain} \newtheorem{glickprop2}[ctsgame]{Lemma}
\theoremstyle{plain} \newtheorem{glickprop3}[ctsgame]{Lemma}

\theoremstyle{plain} \newtheorem{aroralemma1}[ctsgame]{Lemma}
\theoremstyle{plain} \newtheorem{aroralemma2}[ctsgame]{Lemma}
\theoremstyle{plain} \newtheorem{aroralemma3}[ctsgame]{Lemma}
\theoremstyle{plain} \newtheorem{arorathmrestated}[ctsgame]{Theorem}

\newcommand{\ev}[2]{\mathbb{E}_{#1\sim #2}}
\newcommand{\gdist}{p_G}
\newcommand{\rdist}{p_r}
\newcommand{\ndist}{p_z}
\newcommand{\eucl}[1]{\mathbb{R}^{#1}}
\newcommand{\val}[1]{V_{\text{{\normalfont #1}}}}
\newcommand{\ganobj}[3]{\ev{x}{\rdist}[#1(#2(x))]+\ev{z}{\ndist}[#1(1-#2(#3(z)))]}
\newcommand{\intoverofwrt}[3]{\int_{#1} #2\ \mathrm{d}#3}
\newcommand{\dive}[3]{d_{\text{{\normalfont #1}}}\left(#2 \| #3\right)}

\newcommand{\probspace}[1]{\text{{\normalfont Prob}}(#1)}
\newcommand{\ctsbdd}[1]{\mathcal{C}_b(#1)}
\newcommand{\ipm}[3]{\sup_{f\in #1}\biggl( \ev{x}{#2}[f(x)]-\ev{y}{#3}[f(y)] \biggr)}
\newcommand{\loss}[2]{J^{#1}_{\text{{\normalfont #2}}}}
\newcommand{\tvnorm}[1]{\lVert #1 \rVert_{\text{{\normalfont TV}}}}

\title{
	{Convergence Problems with Generative Adversarial Networks}\\
	{\large University of Oxford}\\
	{\includegraphics{oxfordlogo.png}}
}
\author{Candidate Number: 563967}
\date{}

\begin{document}
\begin{titlepage}
    \begin{center}
        \vspace*{1cm}
        
        \Huge
        \textbf{Convergence Problems with Generative Adversarial Networks (GANs)}
        
        \vspace{0.5cm}
        \LARGE
        
        \vspace{1.5cm}
        
        \textbf{S. A. Barnett} \\
        \textbf{MMathPhil Mathematics and Philosophy}
        
        \vfill
        
        A dissertation presented for\\
        CCD Dissertations on a Mathematical Topic
        
        \vspace{0.8cm}
        
        \includegraphics[width=0.4\textwidth]{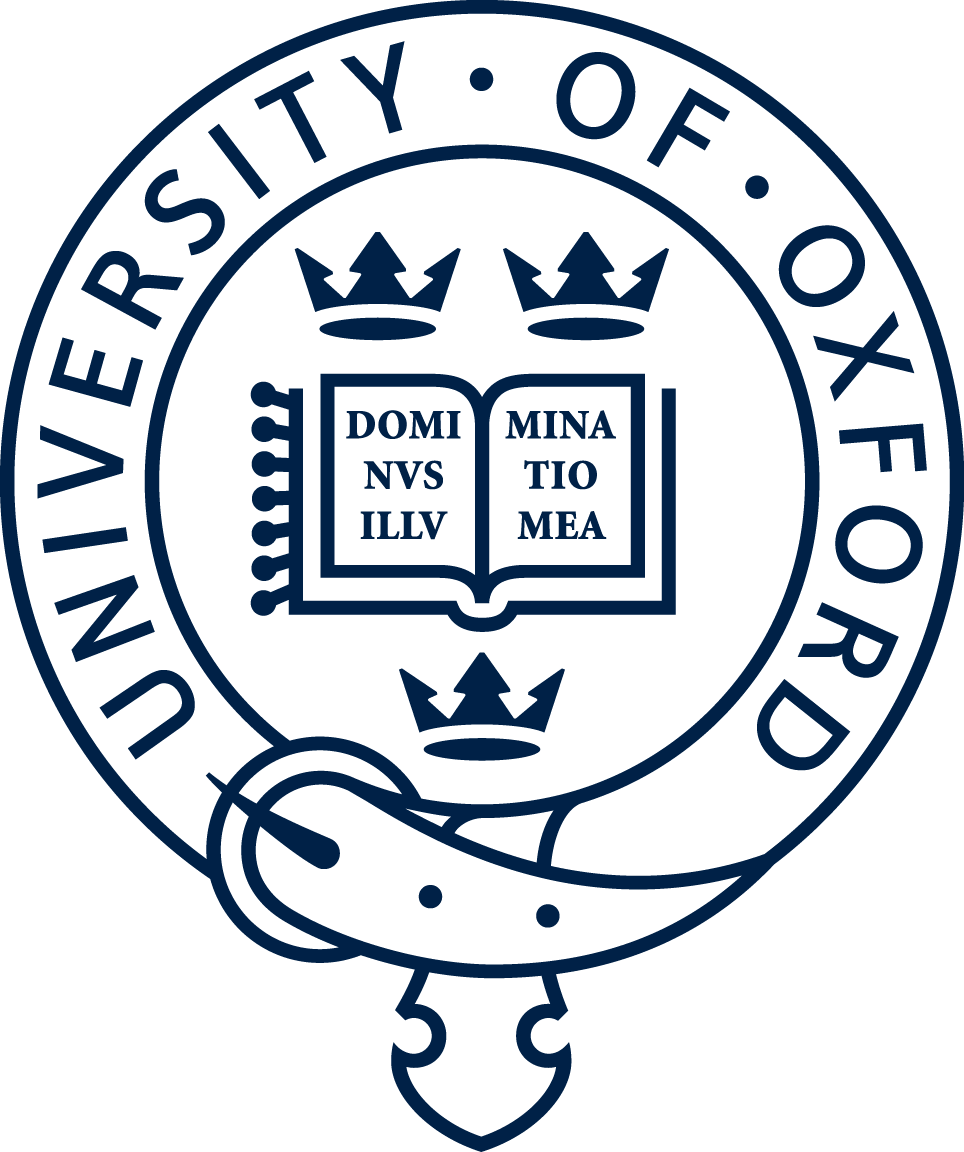}
        
        \Large
        Mathematical Institute\\
        University of Oxford\\
        Hilary Term 2018
        
    \end{center}
\end{titlepage}

\begin{abstract}
Generative adversarial networks (GANs) are a novel approach to generative modelling, a task whose goal it is to learn a distribution of real data points. They have often proved difficult to train: GANs are unlike many techniques in machine learning, in that they are best described as a two-player game between a discriminator and generator. This has yielded both unreliability in the training process, and a general lack of understanding as to \textit{how} GANs converge, and if so, to \textit{what}. The purpose of this dissertation is to provide an account of the theory of GANs suitable for the mathematician, highlighting both positive and negative results. This involves identifying the problems when training GANs, and how topological and game-theoretic perspectives of GANs have contributed to our understanding and improved our techniques in recent years.
\end{abstract}

\renewcommand{\abstractname}{Acknowledgements}
\begin{abstract}
This work was originally presented for the Part C Dissertations on a Mathematical Topic in the MMathPhil Mathematics and Philosophy course at the University of Oxford. As such, I would like to first and foremost thank my supervisor Varun Kanade for his constructive feedback and support in writing this. I would also like to thank my tutor at Worcester College, Richard Earl, for the support during this project and throughout my time at Oxford.
\end{abstract}

\tableofcontents

\chapter{Introduction}
Generative adversarial networks (GANs) were proposed by \citet{goodfellow2014generative} as a novel approach to generative modelling, a task whose goal it is to learn a distribution of real data points.

The term \textit{adversarial} refers to the use of two opposing neural networks in GANs: a \textit{discriminator} trained to tell real data samples apart from GAN-produced samples, and a \textit{generator} that seeks to fool the discriminator. As can be seen in Figure~\ref{fig:prog_samples}, GANs are capable of producing stunningly realistic samples.

\begin{figure}
	\includegraphics[width=\linewidth]{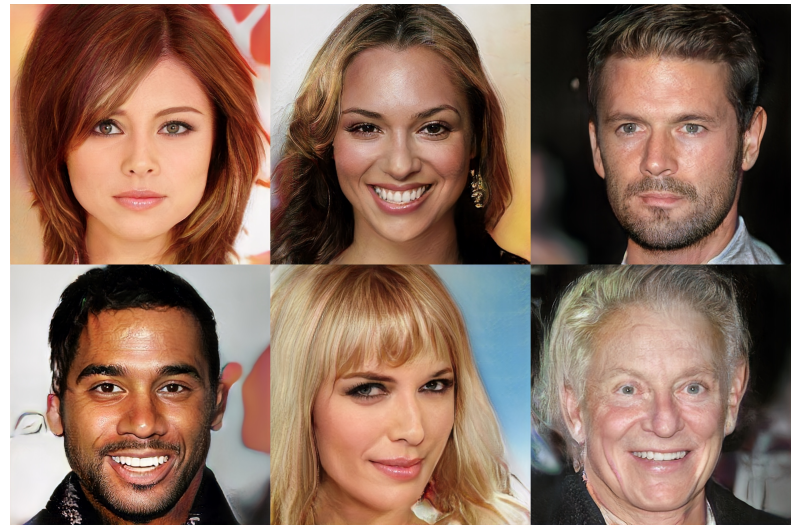}
	\caption{$1024\times1024$ images generated on the CelebA-HQ dataset. Image taken from \citet[Figure 5]{karras2017progressive}.}
	\label{fig:prog_samples}
\end{figure}

However, they have also proved difficult to train: GANs are unlike many techniques in machine learning, in that they are best described as a two-player game between a discriminator and generator. This has yielded both unreliability in the training process, and a general lack of understanding as to \textit{how} GANs converge, and if so, to \textit{what}.

The purpose of this dissertation is to provide an account of the theory of GANs suitable for the mathematician, highlighting both positive and negative results. Chapter 2 introduces GANs in their original formulation, in addition to some of the main problems encountered during the training process. The results of this chapter are largely due to \citet{goodfellow2014generative}, though I provide for the first time an explicit proof\footnote{The result has been claimed, though not proven, by a number of authors \citep{goodfellow2014generative, goodfellow2016nips, metz2016unrolled, arjovsky2017wasserstein}.} of Proposition~\ref{optgenfact}, and give a novel example (Corollary~\ref{cor:optgen}) of its negative consequences for the training of GANs. Chapter 3 gives a perspective of GANs as minimising some divergence between the generator distribution and the target distribution, arguing that certain variants of GANs may induce a more practically useful notion of divergence than that induced by the original GAN formulation. The results of this chapter are predominately due to \citet{arjovsky2017wasserstein}. Chapter 4 discusses GANs from the perspective of game theory, which allows for a broader modelling of GAN training dynamics than that of the previous chapter. The change in emphasis is inspired by \citet{fedus2017many}, with the main result for GANs coming from \citet{arora2017generalization}. Chapter 5 concludes this work.

\section{Notation}

\begin{itemize}
\item $\mu\otimes\nu$ - the product measure, for measures $\mu$ and $\nu$.
\item $\mathcal{B}(\mathcal{X})$ - the space of Borel-measurable subsets of $\mathcal{X}$.
\item $\mathrm{Prob}(\mathcal{X})$ - the space of probability measures on $\mathcal{X}$.
\item $\ctsbdd{\mathcal{X}}$ - the space of bounded, continuous functions from $\mathcal{X}$ to $\mathbb{R}$.
\item $[N]$ - the set of integers $\{1, ..., N\}$, where $N\in\mathbb{N}$.
\item $\mathcal{N}(\mu, \sigma^2)$ - the Gaussian distribution on $\mathbb{R}^n$ with mean $\mu$ and variance $\sigma^2$.
\end{itemize}

\chapter{GANs: Initial Results}
\section{Motivation}

The goal of generative modelling is to learn a particular distribution $\rdist$ of real data. The distribution $\rdist$ may be represented explicitly\footnote{For a taxonomy of such methods, see \citet[Sections 2.2-2.4]{goodfellow2016nips}.}, or \textit{implicitly} by providing a means to produce samples from the distribution.

The latter approach is taken by \textbf{generative adversarial networks (GANs)}.
We may view a GAN as a game between two players: a \textbf{generator} and a \textbf{discriminator}. The former is represented by a function $G$ inducing a distribution $\gdist$ (see below), and the latter by a function $D$. 

Consider a GAN trained on images of people. Given a fixed generator, the discriminator is trained to distinguish between images produced by the real dataset (labelled 1) and images produced by the generator (labelled 0). It does so by mapping each data point $x$ to a value in $[0,1]$. In some sense, $D(x)$ represents the probability that $x$ was a real sample rather than a generated one. 

The goal of the generator, consequently is to produce images that the discriminator will classify as being real, while the goal of the discriminator is to classify these same images as in fact being produced by the generator. The generator induces a distribution $\gdist$ by taking a sample $z$ from a (typically Gaussian) prior on input noise variables, $\ndist$, and mapping it to a synthetic data point $G(z)$. Therefore, if $z\sim\ndist$, then $\gdist$ is the distribution such that $G(z)\sim\gdist$. 

The generator aims to produce points such that $D(G(z))$ is closer to 1. In other words, the generator is trying to fool the discriminator. As we shall see, it will succeed at doing so when the images it produces arise from the same probability distribution as that of the real images of people.

\section{Minimax-GANs}

To formally specify a GAN, we need to give to the generator and the discriminator an objective that each seeks to optimise. Though we may give the discriminator and generator distinct objectives, it is common and often useful for there to be one objective that the generator seeks to \textbf{mini}mise, and that the discriminator seeks to \textbf{max}imise. In this case, which I shall refer to as the \textbf{minimax}\footnote{This minimax perspective will be elaborated in the context of game theory in Chapter 4.} case, we can represent the objective by a single \textbf{value function} $V(D, G)$.

\begin{idealmmgan}[Idealised MM-GANs \citep{goodfellow2014generative}]
Let $\mathcal{Z}\subseteq\eucl{\ell}$, $\mathcal{X}\subseteq\eucl{d}$ be ambient data spaces, let $\ndist$ be a prior distribution over $\mathcal{Z}$, and let $\rdist$ be the distribution of real data points over $\mathcal{X}$. The \textbf{idealised minimax GAN} (IMM-GAN) is the game specified by the objective
\begin{equation}\label{immobj}
\min_G\max_D\val{IMM}(D, G),
\end{equation}
where $G\colon\mathcal{Z}\to\mathcal{X}$, $D\colon\mathcal{X}\to[0,1]$, and
\begin{equation}
\val{IMM}(D, G) = \ganobj{\log}{D}{G}.
\end{equation}
\end{idealmmgan}
\begin{idealremark}
We observe that $\val{IMM}$ is the sum of two expected-value terms. The first of these captures the idea that the discriminator wants to mark with high probability points from the real data set. The second of these terms captures that the discriminator also wants to mark with low probability points from the synthetic data set. It also captures the contrasting goal of the generator, which is to fool the discriminator into marking synthetic points with high probability. 
\end{idealremark}
\begin{idealremark2}
This set-up is \textbf{idealised}, in that it searches for the optimal $D$ and $G$ over the space of \textit{all} functions with the correct domain and co-domain. 
\end{idealremark2}

Using this specification of the GAN game, we wish to show that the objective function is met precisely when $\rdist=\gdist$. If this is the case, then a GAN is successfully trained if and only if the generator distribution matches the target distribution: this is precisely what we require of GANs. To show this, we first give a result specifying the optimal discriminator, given a fixed generator.

\begin{optdim}[\citet{goodfellow2014generative}, Proposition 1]
Fix $G$ in the IMM-GAN game. The optimal discriminator $D$ as required by the maximisation term in~\eqref{immobj} is given by
\begin{equation}\label{optdisc}
D^*(x) = \frac{\rdist(x)}{\rdist(x) + \gdist(x)}.
\end{equation}
\end{optdim}
\begin{proof}
Maximising $\val{IMM}$ with respect to $D$ is equivalent to maximising 
\begin{align*}
\val{IMM}(D, G) &= \ganobj{\log}{D}{G} \\
&= \intoverofwrt{\mathcal{X}}{\rdist(x)\log(D(x))}{x} +\intoverofwrt{\mathcal{Z}}{\ndist(z)\log(1-D(G(z)))}{z} \\
&= \intoverofwrt{\mathcal{X}}{\bigl[\rdist(x)\log(D(x))+\gdist(x)\log(1-D(x))\bigr]}{x}.
\end{align*}
For any $(a,b)\in\mathbb{R}^2\setminus\{(0,0)\}$, the function $y\mapsto a\log y + b\log(1-y)$ achieves its maximum in $[0,1]$ at $\frac{a}{a+b}$. Since the discriminator does not need to be defined outside of the values of $x$ for which $\rdist$ and $\gdist$ are non-zero, this concludes the proof.
\end{proof}
\begin{optdimcor}[\citet{goodfellow2014generative}]\label{optdisccor}
The IMM-GAN game is equivalent to finding
\begin{equation}
\min_G C(G),
\end{equation}
where $G\colon\mathcal{Z}\to\mathcal{X}$, and
\begin{equation}
C(G) = \ev{x}{\rdist}\left[\log\frac{\rdist(x)}{\rdist(x)+\gdist(x)}\right] +\ev{x}{\gdist}\left[\log\frac{\gdist(x)}{\rdist(x)+\gdist(x)}\right].
\end{equation}
\end{optdimcor}

\section{Divergences for GANs}

To show that the IMM-GAN objective function makes sense, we need to introduce one last important element: the notion of a \textit{divergence} between two probability distributions. This is akin to a measure of distance between two distributions: if we minimise the divergence, we also hope that the two distributions are in fact equal. Divergences come up again in the next chapter, but for now it suffices to consider two possible definitions of divergence.

\begin{kljsdef}
Let $\mu$ and $\nu$ be two probability measures, and suppose $\mu$ is absolutely continuous with respect to $\nu$. The \textbf{Kullback-Leibler (KL) divergence} from $\nu$ to $\mu$ is defined as
\begin{equation}
\dive{KL}{\mu}{\nu} = \ev{x}{\mu}\left[\log\frac{\mu(x)}{\nu(x)}\right] = \ev{x}{\mu}[\log \mu(x) - \log \nu(x)].
\end{equation}
The \textbf{Jensen-Shannon (JS) divergence} is defined as
\begin{equation}
\dive{JS}{\mu}{\nu} = \frac{1}{2}\dive{KL}{\mu}{M} + \frac{1}{2}\dive{KL}{\nu}{M},
\end{equation}
where $M = (\mu+\nu)/2$.
\end{kljsdef}

Two important properties of these divergences make them useful as a notion of the difference between two distributions.\footnote{The KL divergence is not a distance function as it is not symmetric - it is possible that $\dive{KL}{\mu}{\nu}\neq\dive{KL}{\nu}{\mu}.$ For an example of this, see \citet[Figure 14]{goodfellow2016nips}. The JS divergence is a symmetrised version of the KL divergence.}

\begin{kljsprop}[\citet{kullback1951information}, Lemma 3.1]
Let $\mu, \nu$ be two distributions for which the KL divergence is defined. Then $\dive{KL}{\mu}{\nu}$ is non-negative, and equal to 0 if and only if $\mu$ and $\nu$ are equal almost everywhere. If $\mu$ and $\nu$ are discrete probability distributions, this is equivalent to $\mu$ being equal to $\nu$.
\end{kljsprop}
\begin{proof}
This proof relies on $\log$ being concave, and $-\log$ thus being convex. Consider the case in which $\mu$ and $\nu$ are continuous probability distributions (the proof works for discrete distributions \textit{mutatis mutandis}). Then, by Jensen's inequality:
\begin{align*}
\dive{KL}{\mu}{\nu} &= \ev{x}{\mu}\left[\log\frac{\mu(x)}{\nu(x)}\right] \\
&= \int -\mu(x)\log\frac{\nu(x)}{\mu(x)} \mathrm{d}x \\
&\ge -\log\left(\int \mu(x)\cdot\frac{\nu(x)}{\mu(x)} \mathrm{d}x\right) \\
&= -\log(1) = 0.
\end{align*}
Hence, $\dive{KL}{\mu}{\nu}\ge0$. If $\mu=\nu$ almost everywhere, then it is clear from the definition that $\dive{KL}{\mu}{\nu} = 0$. Moreover, since $\log$ is a \textit{strictly} convex function, then the weak inequality is an equality only if $\mu=\nu$ almost everywhere.
\end{proof}

\begin{kljscor}
Let $\mu, \nu$ be two distributions for which the JS divergence is defined. Then $\dive{JS}{\mu}{\nu}$ is non-negative, and equal to 0 if and only if $\mu$ and $\nu$ are equal almost everywhere. If $\mu$ and $\nu$ are discrete probability distributions, this is equivalent to $\mu$ being equal to $\nu$.
\end{kljscor}
\begin{proof}
This follows from the above proposition, and noting that the JS divergence is defined as the sum of two (non-negative) KL divergences.
\end{proof}

\section{Appropriateness of Objective Function}

The following theorem shows that the choice of objective function for the IMM-GAN is well-motivated.

\begin{immsol}[\citet{goodfellow2014generative}, Theorem 1]\label{thm:immsol}
The global minimum of the training criterion $C(G)$ is achieved if and only if $\rdist=\gdist$. At that point, $C(G)$ achieves the value $-\log4$.
\end{immsol}
\begin{proof}
For $\rdist=\gdist$, \eqref{optdisc} gives us that $D^*(x) = \frac{1}{2}$, so that $C(G) = \log\frac{1}{2} + \log\frac{1}{2} = -\log4$. To see that this is the minimal value of $C(G)$, reached only for $\rdist=\gdist$, observe that 
\[\ev{x}{\rdist}[-\log2] + \ev{x}{\gdist}[-\log2] = -\log4.\]
By subtracting this expression from $C(G) = V(D^*, G)$, we obtain:
\begin{align*}
C(G) &= -\log4 + \dive{KL}{\rdist}{\frac{p_{data}+p_G}{2}} + \dive{KL}{\gdist}{\frac{\rdist+\gdist}{2}} \\
&= -\log4 + 2\cdot \dive{JS}{\rdist}{\gdist}.
\end{align*}

Since the JS divergence between two distributions is always non-negative and zero only when the distributions are equal, we have that $C^* = -\log4$ is the global minimum of $C(G)$ whose only solution is $\rdist=\gdist$.
\end{proof}

Using this proof, we may also establish that in the ideal case in which we may make updates within the function space, a broad class of convex optimisation algorithms may find this unique solution.

\begin{immconv}[Adapted from \citet{goodfellow2014generative}, Proposition 2]\label{thm:immconv}
The function \[
U(\gdist, D) = \ev{x}{\rdist}[\log D(x)]+\ev{x}{\gdist}[\log(1-D(x))].
\] is convex in $\gdist$.
\end{immconv}
\begin{proof}
We observe that only the second term depends on $\gdist$. The proof then follows from the linearity of expectation.
\end{proof}

\section{Practical Implementation of GANs}

\begin{figure}
	\includegraphics[width=\linewidth]{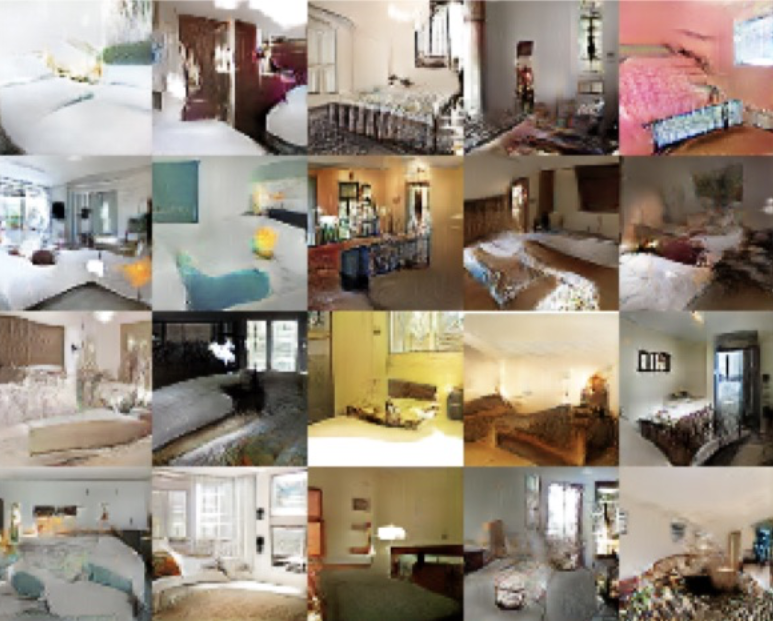}
	\caption{Samples of images of bedrooms generated by a GAN trained on the LSUN dataset, taken from \citet{goodfellow2016nips}.}
	\label{fig:DCGAN_samples}
\end{figure}

In practice we do not search over all possible functions $G$ and $D$ for our optima. Instead, we consider a family of parametrised functions $G(z; \theta_G)$ and $D(x; \theta_D)$ and optimise parameters $\theta_G\in\Theta_G$ and $\theta_D\in\Theta_D$. The typical class of parametrised functions we consider are \textbf{neural networks}, often abbreviated to \textbf{neural nets}. 

A precise formalisation of neural nets is beyond the scope of this dissertation: the interested reader is referred to \citet[Chapter 6]{goodfellow2016deep}. It suffices to know that neural nets have the power to approximate a broad class of functions, and are differentiable with respect to their defining parameters. The latter fact means that an objective function defined in terms of a neural net may be maximised (resp. minimised) by taking steps in the parameter space proportional to the \textit{negative} (resp. the \textit{positive}) of the gradient, in a process referred to as \textit{gradient descent}.\footnote{The objective-maximising equivalent is also referred to as \textit{gradient ascent}.}

Restricting our generator and discriminator to be neural nets allows for the GAN to be implemented and trained in practice.

\begin{algorithm}
\caption{Minibatch stochastic gradient descent training of MM-GANs \citep{goodfellow2014generative}. The gradient-based updates can use any standard gradient-based learning rule.}
\begin{algorithmic}[1]
\FOR{number of training iterations}
	\FOR{$k$ steps}
		\STATE Sample minibatch of $m$ noise samples $\bigl\{z^{(1)},...,z^{(m)}\bigr\}$ from noise prior $\ndist(z)$.
		\STATE Sample minibatch of $m$ examples $\bigl\{x^{(1)},...,x^{(m)}\bigr\}$ from real data distribution $\rdist$.
		\STATE Update the discriminator by ascending its stochastic gradient: \[
		\nabla_{\theta_D}\frac{1}{m}\sum_{i=1}^m\bigl[\log D\bigl(x^{(i)}\bigr)+\log\bigl(1-D\bigl(G\bigl(z^{(i)}\bigr)\bigr)\bigr)\bigr].
		\]
	\ENDFOR
	\STATE Sample minibatch of $m$ noise samples $\bigl\{z^{(1)},...,z^{(m)}\bigr\}$ from noise prior $\ndist(z)$.
	\STATE Update the generator by descending its stochastic gradient: \[
		\nabla_{\theta_G}\frac{1}{m}\sum_{i=1}^m\log\bigl(1-D\bigl(G\bigl(z^{(i)}\bigr)\bigr)\bigr).
		\]
\ENDFOR
\end{algorithmic}
\end{algorithm}

We formalise the new objective as follows:
\begin{mmgan}[MM-GAN]
Let $\mathcal{Z}\subseteq\eucl{\ell}$, $\mathcal{X}\subseteq\eucl{d}$ be ambient data spaces, let $\ndist(z)$ be a prior distribution over $\mathcal{Z}$, and let $\rdist$ be the distribution of real data points over $\mathcal{X}$. Let $\Theta_D$ and $\Theta_G$ be the spaces of possible parameters for the discriminator and generator, respectively.\footnote{Typically, $\Theta_D$ and $\Theta_G$ are subsets of the unit ball.} The \textbf{minimax GAN} (MM-GAN) is the game specified by the objective
\begin{equation}\label{mmobj}
\min_{\theta_G\in\Theta_G}\max_{\theta_D\in\Theta_D}\val{MM}(D_{\theta_D}, G_{\theta_G}),
\end{equation}
where $G_{\theta_G}\colon\mathcal{Z}\to\mathcal{X}$, $D_{\theta_D}\colon\mathcal{X}\to[0,1]$ belong to classes of neural nets
\begin{align*}
\mathcal{F} &= \{G_{\theta_G}\mid \theta_G\in\Theta_G\}, \\
\mathcal{G} &= \{D_{\theta_D}\mid \theta_D\in\Theta_D\},
\end{align*} and
\begin{equation}
\val{MM}(D_{\theta_D}, G_{\theta_G}) = \ganobj{\log}{D_{\theta_D}}{G_{\theta_D}}.
\end{equation}
\end{mmgan}

\section{Convergence Problems}

Unlike its idealised counterpart, the MM-GAN objective function lacks counterparts to Theorems~\ref{thm:immsol} and~\ref{thm:immconv} that guarantee convergence to a unique solution such that $\rdist = \gdist$. This section reviews two particular problems with convergence observed when implementing this GANs in practice, and considers the theoretical explanations of their origins that have been offered. The remainder of this dissertation will focus on theoretically-motivated modifications of GANs that seek to ameliorate these problems.

\subsection{Failure to Improve}
There are two ways in which our generator may fail to improve, where the \textit{improvement} is taken with respect to the quality of samples it produces. In the first case, though a solution may exist, the dynamics of the gradient descent training algorithm prevents the neural nets from reaching their optimal parameter values. Example~\ref{ex:cycles} demonstrates this.

In the second case, which is a special case of the first failure, the gradient along which the generator must train is diminished to the point that the generator cannot usefully learn from it. This is known as the \textbf{vanishing gradient problem}, or the \textbf{saturation problem}. \citet{goodfellow2014generative} claims that this problem is caused by the discriminator successfully rejecting generator samples with high confidence, so that the generator's gradient vanishes. This suggests that we ought to avoid over-training the discriminator, and instead carefully interplay discriminator and generator improvements.

\subsection{Mode Collapse}

Mode collapse is a problem that occurs when the generator learns to produce only a limited range of samples from the real data distribution. It does so by mapping several different input values $z\sim\ndist$ to the same output point $G(z)$. 

The name `mode collapse' comes from the fact that, when trying to learn a multi-modal distribution, the generator only outputs samples from a select number of these modes. \citet{metz2016unrolled} demonstrates this by showing how a GAN may fail to learn a toy data distribution consisting of a mixture of 2D Gaussian distributions (as seen in Figure~\ref{fig:Metz_modecollapse}).

\begin{figure}
	\includegraphics[width=\linewidth]{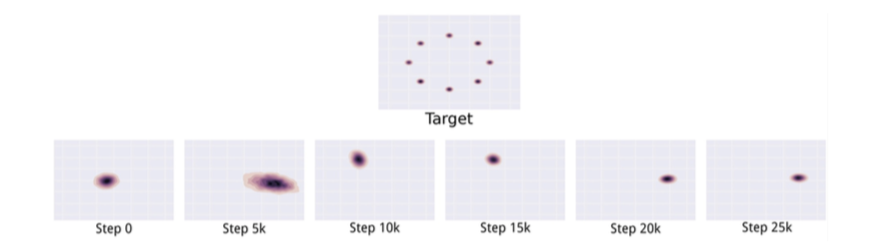}
	\caption{An illustration of mode collapse on a toy dataset consisting of a mixture of Gaussians in two-dimensional space. In the bottom row, we see how as the GAN is trained over time, the generator only produces a single mode at a time, cycling between different modes as the discriminator learns to reject each one. Image taken from \citet{metz2016unrolled}.}
	\label{fig:Metz_modecollapse}
\end{figure}

\citet{goodfellow2014generative} and \citet{metz2016unrolled} postulate that mode collapse arises from the following fact, which I shall state and prove rigorously:

\begin{optgen}\label{optgenfact}
Fix a continuous $D$ in the IMM-GAN game, and let $\mathcal{X}$ be compact.\footnote{The assumption of compactness for our data space makes sense in a practical context. By the Heine-Borel Theorem, a subset of Euclidean space is compact if and only if it is closed and bounded. Representations of real data often take this form: for example, a grayscale image can be given by a finite-dimensional vector of values in $[0,1]$, and so the space of all such grayscale images will be compact. Hence, compactness is assumed here and elsewhere without loss of practical generality.} The optimal generator $G$ as required by the outer loop of~\eqref{immobj} is given by
\begin{equation}\label{optgene}
G^*(z) = s_z \quad \forall z\in\mathcal{Z},
\end{equation}
where $s_z\in\arg\max_{x\in\mathcal{X}}D(x).$ In other words, the optimal generator for a fixed discriminator maps every value $z\in\mathcal{Z}$ to some $x\in\mathcal{X}$ that the discriminator believes is most likely to be real rather than fake.
\end{optgen}
\begin{proof}
Since the discriminator is continuous with a compact domain, the set \hfill \break $\arg\max_{x\in\mathcal{X}}D(x)$ is non-empty. We observe that only the second term in the value function \[
\ganobj{\log}{D}{G}
\] depends on $G$. Since $\log$ is a monotone function, we wish to choose $G$ so as to minimise $(1-D(G(z)))$, or equivalently, so as to maximise $D(G(z))$. From this the statement follows.
\end{proof}

\begin{optgencor}\label{cor:optgen}
Define $\tilde{C}(D) := \min_G\val{IMM}(D, G)$. Let $\rdist$ be any real data distribution with support on $\mathcal{X}$ compact. Then there exists a discriminator $D$ such that $\tilde{C}(D) = \min_G C(G)$, but $\rdist\neq\gdist$.
\end{optgencor}
\begin{proof}
By the above result, 
\begin{equation}
\tilde{C}(D) = \ev{x}{\rdist}[\log D(x)] +\log(1-\max_{y\in\mathcal{X}}D(y)).
\end{equation}
Taking $D$ to be constantly 1/2, we get that $\tilde{C}(D) = -\log4 = \min_G C(G)$. However, this value for $\tilde{C}(D)$ can be attained for any $G$ whose value is constantly $\arg\max_{y\in\mathcal{X}}D(y)$. In particular, it can be attained for a $G$ such that $\rdist\neq\gdist$.
\end{proof}

Suppose we viewed the objective of MM-GAN as finding \[
\max_{\theta_D\in\Theta_D}\min_{\theta_G\in\Theta_G}\val{MM}(D_{\theta_D}, G_{\theta_G}).
\] By the above proposition, this approach seemingly encourages a scenario in which the generator favours producing only one output. The problem with the GAN training algorithm, according to \citet{goodfellow2016nips}, is that it does not demonstrate any preference over the \textbf{maximin} and \textbf{minimax} perspective. As emphasised by \citet{arjovsky2017wasserstein}, this suggests we should seek to train the discriminator to optimality before each step of generator training. 

However, this runs counter to the advice given to resolve the problem of convergence failure. As such, we need to modify our GAN design so that we may train the discriminator to optimality, avoiding the issue of mode collapse, while at the same time avoiding convergence failures.

\chapter{The Topology of GANs}
In the previous chapter, we saw that the original GAN \citep{goodfellow2014generative} formulation suffered from problems of \textbf{convergence failure} and \textbf{mode collapse} during the training procedure. In this chapter, I review a generalisation of the GAN objective function due to \citet{liu2017approximation} and \citet{zhang2017discrimination}. This generalisation gives us a deeper theoretical insight into the conditions that must be satisfied for a GAN to be able to successfully reproduce samples from a distribution.

Recall that, fixing an optimal discriminator and allowing updates to the function space, finding an optimal generator is equivalent to minimising the Jensen-Shannon (JS) divergence between the generator distribution $\gdist$ and the real data distribution, $\rdist$. The generalisation in this chapter, taking this as inspiration, shows how changing our choice of objective function makes finding the optima of that function equivalent to minimising some divergence between the two distributions.

Of course, convergence depends on our choice of distance or divergence $\rho(p_\theta, \rdist)$ between these distributions. This chapter develops the argument in \citet{arjovsky2017wasserstein} that GAN training demands a distance notion $d_{\text{W}}$ that induces a \textit{weaker topology} than $d_{\text{JS}}$, in that the set of convergent sequences under $d_{\text{W}}$ will be a superset of that under $d_{\text{JS}}$. I shall then show the positive theoretical results of the corresponding GAN procedure, \textbf{Wasserstein GAN} (WGAN).

\section{Adversarial Divergences}

In Chapter 2, we established that, given an optimal discriminator, we can view the IMM-GAN game as a minimisation problem for the generator. In particular, Corollary~\ref{optdisccor} showed that the IMM-GAN game was equivalent to finding
\begin{equation*}
\min_G\biggl( \ev{x}{\rdist}\left[\log\frac{\rdist(x)}{\rdist(x)+\gdist(x)}\right] +\ev{x}{\gdist}\left[\log\frac{\gdist(x)}{\rdist(x)+\gdist(x)}\right]\biggr).
\end{equation*}

\citet{liu2017approximation} generalises this approach, viewing a GAN as seeking to minimise the objective function
\begin{equation}
\gdist\mapsto\sup_{f\in\mathcal{F}}\mathbb{E}_{x\sim\rdist, y\sim\gdist}[f(x,y)]
\end{equation}
for some class $\mathcal{F}$ of functions.\footnote{That our objective function is defined on a \textit{distribution space} rather than a \textit{parameter space} shows that this approach is `idealised' in the sense given in the previous chapter.} This leads to the concept of \textit{adversarial divergence}.

\begin{advdiv}[Modified from \citet{liu2017approximation}, Definition 1]
Let $\mathcal{X}$ be an arbitrary topological space, $\mathcal{F}\subseteq [0,1]^{(\mathcal{X}^2)}$ our class of functions with domain $\mathcal{X}^2$. An \textbf{adversarial divergence} $d_\tau$ over $\mathcal{X}$ is a function
\begin{align*}
\probspace{\mathcal{X}}\times\probspace{\mathcal{X}} &\to \mathbb{R}\cup\{+\infty\} \\
(\mu,\nu) &\mapsto d_\tau(\mu \| \nu) =: \sup_{f\in\mathcal{F}}\mathbb{E}_{\mu\otimes\nu}[f].
\end{align*}
\end{advdiv}

\begin{immadvdiv}[IMM-GAN \citep{goodfellow2014generative}]
If we set
\begin{equation}
\mathcal{F} = \{x,y\mapsto\log (D(x))+\log (1-D(y))\mid D\in\mathcal{V}\},
\end{equation}
where $\mathcal{V} = [0,1]^\mathcal{X}$, we recover our IMM-GAN objective with optimal discriminator, $C(G)$.
\end{immadvdiv}

\begin{ipmadvdiv}[Integral Probability Metric \citep{muller1997integral}]
We derive a particularly important class of GANs when we assume that, in the definition of adversarial divergence, we can write our bivariate function $f$ as the difference of two univariate functions. 

In particular, given a choice\footnote{Refer to \citet{zhang2017discrimination} for an exploration of the consequences of our choice to $\mathcal{F}$ on how \textit{useful} the consequent metric is, as well as the extent to which the empirical error bounds will \textit{generalise} to true error bounds.} of $\mathcal{F}$, an \textbf{integral probability metric} (IPM) between two distributions is defined
\begin{equation}
\dive{IPM}{\mu}{\nu} := \ipm{\mathcal{F}}{\mu}{\nu}.
\end{equation}
\end{ipmadvdiv}

\begin{ipmprop}\label{ipmpseudometric}
Suppose that our function class $\mathcal{F}$ is such that, if $f\in\mathcal{F}$, then $-f\in\mathcal{F}$. Then $\dive{IPM}{\mu}{\nu}$ is non-negative, satisfies the triangle inequality, and is symmetric.
\end{ipmprop}
\begin{proof}
The proof follows easily from the properties of the supremum.
\end{proof}

\section{Wasserstein GAN}

This section defines the \textbf{Wasserstein GAN} (WGAN), which can be shown to arise from a particular choice of IPM. WGAN was originally developed by \citet{arjovsky2017wasserstein}, after being theoretically motivated in \citet{arjovsky2017towards}. The theory has been developed further by \citet{bousquet2017optimal} and \citet{lei2017geometric}.

\subsection{Earth-Mover Distance and Total Variation Distance}

We first define two notions of distance between probability distributions, both of which can be shown to be examples of IPMs.

\begin{w1}
Let $\mu$, $\nu$ be probability measures on a compact metric space $(\mathcal{X}, d)$. The \textbf{Earth-Mover} (EM) or \textbf{Wasserstein-1} distance is given by
\begin{equation}\label{emdist}
\dive{W}{\mu}{\nu} := \inf_{\gamma\in\Pi(\mu,\nu)}\mathbb{E}_{(x,y)\sim\gamma}[\|x-y\|],
\end{equation}
where $\Pi(\mu,\nu)$ denotes the set of all joint distributions $\gamma(x,y)$ such that, for all $A\in\mathcal{B}(\mathcal{X})$,
\begin{align*}
\gamma(A, \mathcal{X}) &= \mu(A), \\
\gamma(\mathcal{X}, A) &= \nu(A).
\end{align*}
Intuitively, $\gamma(x,y)$ indicates how much `mass' must be transported from $x$ to $y$ in order to transform the distribution $\mu$ into the distribution $\nu$.
\end{w1}

The following equivalent formula is more tractable when finding minima with respect to the Wasserstein distance.

\begin{krduality}[The Kantorovich-Rubinstein Duality]\label{krdualthm}
Let $(\mathcal{X}, d)$ be a compact metric space, and let $\text{{\normalfont Lip}}_1(\mathcal{X})$ be the set of functions $f\colon\mathcal{X}\to\mathbb{R}$ such that \[
\|f\|_{\text{{\normalfont L}}} := \sup\left\{\frac{|f(x)-f(y)|}{d(x,y)}\ \bigg\vert\  x,y\in\mathcal{X}, x\neq y\right\} \le 1.
\]
Then $f\in\text{{\normalfont Lip}}_1(\mathcal{X})$ is Lebesgue integrable with respect to any probability measure on $\mathcal{X}$, and
\begin{equation}\label{krd}
\dive{W}{\mu}{\nu} = \sup_{f\in\text{{\normalfont Lip}}_1(\mathcal{X})} (\mathbb{E}_{x\sim\mu}[f(x)] - \mathbb{E}_{x\sim\nu}[f(x)]).
\end{equation}
\end{krduality}
\begin{proof}
The result is a standard one in optimal transport theory. See, e.g., \citet[Theorem 5.10]{villani2008optimal} for a proof. An alternate proof can be found in \citet[Theorem 4.1]{edwards2011kantorovich}. The presentation here is an adaptation of the latter approach, as given by \citet[Theorem 1.3]{basso2015hitchhiker}.

Since $f$ is 1-Lipschitz, we have for some $x_0\in\mathcal{X}$ that 
\begin{equation}\label{lebbound}
\lvert f(x)\rvert \le \lvert f(x_0)\rvert + d(x,x_0).
\end{equation}

Since $\mathcal{X}$ is compact, for any probability measure $\mathbb{P}$ on $\mathcal{X}$ we can integrate both sides of \eqref{lebbound} with respect to $\mathbb{P}$ to obtain \[
\int_\mathcal{X} \lvert f(x)\rvert\ \mathrm{d}\mathbb{P} \le \int_\mathcal{X}\left(\lvert f(x_0)\rvert+d(x,x_0) \right)\ \mathrm{d}\mathbb{P} \le +\infty.
\]

Now let $\mathcal{B}^\infty (\mathcal{X})$ be the set of all bounded Borel-measurable functions $f\colon\mathcal{X}\to\mathbb{R}$. For $f, g\colon\mathcal{X}\to\mathbb{R}$, we define $(f\oplus g)\colon\mathcal{X}^2\to\mathbb{R}$ by \[
(f\oplus g)(x,y) := f(x) + g(y).
\]

Observe that since $\mu, \nu$ are probability measures on a compact metric space, they are (bounded non-negative) Radon measures. Further, $d\colon\mathcal{X}^2\to\mathbb{R}$ is continuous, and hence lower semicontinuous, as a distance metric. By Corollary 3.2 of \citet{edwards2011kantorovich}, we have that
\begin{equation*}
\dive{W}{\mu}{\nu} = \sup\left\{ \int_\mathcal{X} f\ \mathrm{d}\mu + \int_\mathcal{X} g\ \mathrm{d}\nu \mid f,g\in\mathcal{B}^\infty(\mathcal{X}), (f\oplus g)\le d \right\}.
\end{equation*}

Fix $\varepsilon>0$. By the Approximation Lemma, there exists $f,g\in\mathcal{B}^\infty(\mathcal{X})$ with $(f\oplus g)\le d$ such that \[
\dive{W}{\mu}{\nu} - \varepsilon\le \int_\mathcal{X} f\ \mathrm{d}\mu + \int_\mathcal{X} g\ \mathrm{d}\nu.
\]

Now define $k\colon\mathcal{X}\to\mathbb{R}$ by $k(x):= \inf_{y\in\mathcal{X}}\left(d(x,y)-g(y)\right)$. Since $g$ is bounded, $k$ is well-defined. Then, for $x, x'\in\mathcal{X}$, we have that
\begin{align*}
\lvert k(x)-k(x')\rvert &= \lvert \inf_{y\in\mathcal{X}}\left(d(x,y)-g(y)\right) - \inf_{y\in\mathcal{X}}\left(d(x',y)-g(y)\right) \rvert \\
&\le \sup_{y\in\mathcal{X}}\lvert d(x,y)-d(x',y) \rvert \\
&\le d(x,x').
\end{align*}
Hence $k\in\text{Lip}_1(\mathcal{X})$. Note further that, for all $x\in\mathcal{X}$, \[
f(x)\le k(x)\le d(x,x)-g(x)=-g(x),
\] so $f\le k$ and $g\le -k$. Now let $\gamma\in\Pi(\mu, \nu)$. We then get
\begin{align*}
\dive{W}{\mu}{\nu} - \varepsilon &\le \int_\mathcal{X} f\ \mathrm{d}\mu + \int_\mathcal{X} g\ \mathrm{d}\nu \\
&\le \int_\mathcal{X} k\ \mathrm{d}\mu - \int_\mathcal{X} k\ \mathrm{d}\nu \\
&\le \sup\left\{ \int_\mathcal{X} f\ \mathrm{d}\mu - \int_\mathcal{X} f\ \mathrm{d}\nu \mid f\in\text{Lip}_1(\mathcal{X}) \right\} \\
&\le \sup\left\{ \int_{\mathcal{X}\times\mathcal{X}} (f\oplus -f)\ \mathrm{d}\gamma \mid f\in\text{Lip}_1(\mathcal{X}) \right\} \\
&\le \int_{\mathcal{X}\times\mathcal{X}} d(x,y)\ \mathrm{d}\gamma(\mathrm{d}x, \mathrm{d}y).
\end{align*}
Letting $\varepsilon\to0$, we get the desired equality.
\end{proof}

\begin{kripmcor}
The EM distance is an IPM, so long as the domain $\mathcal{X}$ of our function class $\mathcal{F}$ is a compact metric space.
\end{kripmcor}

The following corollary tells us that it makes sense to describe the EM distance as a \textit{distance}, and to talk of it inducing a \textit{topology}.

\begin{krdistcor}[\citet{basso2015hitchhiker}, Corollary 1.4]
Let $(\mathcal{X}, d)$ be a compact metric space. Then $d_{\text{W}}$ defines a metric on $\probspace{\mathcal{X}}$.
\end{krdistcor}
\begin{proof}
That $d_{\text{W}}$ is symmetric and non-negative is clear from \eqref{emdist}, and that it obeys the triangle inequality is clear from \eqref{krd}. Hence it remains to show that for probability measures $\mu, \nu$ on a compact metric space $(\mathcal{X}, d)$ that if $\dive{W}{\mu}{\nu}=0$, then $\mu=\nu$.

Let $F$ be a closed subset of $\mathcal{X}$. For each integer $k\ge1$, we define $f_k\colon\mathcal{X}\to\mathbb{R}$ by $f_k(x):= 1\wedge(k\cdot\text{dist}(x, F))$. Then it follows that, for each integer $k$, $f_k/k\in\text{Lip}_1(\mathcal{X})$. Furthermore, since $\dive{W}{\mu}{\nu}=0$, \eqref{krd} gives us that, for all $k\ge1$,
\begin{equation}\label{mctdist}
\frac{1}{k}\int_\mathcal{X} f_k\ \mathrm{d}\mu = \frac{1}{k}\int_\mathcal{X} f_k\ \mathrm{d}\nu.
\end{equation} 
Observe that $(f_k)_{k\ge1}$ is a non-negative sequence of functions converging monotonically to the indicator function on $\mathcal{X}\setminus F$, an open set. Hence, by the Monotone Convergence Theorem and \eqref{mctdist}, it follows that $\mu(\mathcal{X}\setminus F) = \nu(\mathcal{X}\setminus F)$. Since open subsets of $\mathcal{X}$ generate $\mathcal{B}(\mathcal{X})$, it follows by Dynkin's Lemma that $\mu=\nu$.
\end{proof}

We obtain similar results for the \textbf{Total Variation distance} between two distributions.

\begin{tvdist}
The \textbf{Total Variation} (TV) distance between $\mu$ and $\nu$ is defined
\begin{equation}
\dive{TV}{\mu}{\nu} = \sup_{A\in\mathcal{B}(\mathcal{X})} \lvert \mu(A)-\nu(A)\rvert.
\end{equation}
\end{tvdist}

\begin{tvipmprop}
The TV distance is an IPM, where $\mathcal{F}$ is the set of of all measurable functions bounded between -1 and 1.
\end{tvipmprop}
\begin{proof}
See \citet{muller1997integral}.
\end{proof}

\begin{tvdistcor}
The TV distance defines a metric on $\probspace{\mathcal{X}}$.
\end{tvdistcor}
\begin{proof}
By Corollary \ref{ipmpseudometric}, it suffices to prove that $\dive{TV}{\mu}{\nu} = 0$ implies $\mu=\nu$. But this is given from the definition of $d_{\text{TV}}$ as a supremum: in particular, $\mu$ and $\nu$ are equal on all Borel measurable subsets, and so must be equal.
\end{proof}

\subsection{The Weakness of the Wasserstein Distance}

We now seek to show that, in some rigorous sense, minimising Wasserstein distance is a more suitable framework for GAN training than minimising the Jensen-Shannon divergence. The adversarial divergence framework enables us to view the objective of GANs as the minimisation of some divergence between our generator distribution $\gdist$ and our target distribution $\rdist$. Moreover, the framework can be used to consider the \textit{training} of a GAN as the convergence of $\gdist$ to $\rdist$ with respect to the given divergence.

With distinct definitions of convergence come distinct induced \textit{topologies}.\footnote{Not all divergences define metrics (e.g., the KL and reverse-KL divergence). As a result, a given adversarial divergence may not give us a topology in a strictly formal sense.} The idea of convergence gives rise to an idea of a \textit{topology} induced by the divergence. Notably, we ought to seek some divergence that give us a \textit{weak} topology, in that the convergence of $\gdist$ to $\rdist$ with respect to other divergences implies convergence with respect to our ideal divergence. Using the language of functional analysis, it can be shown that the Wasserstein distance meets this desideratum.\footnote{This argument can be found in \citet[Appendix A]{arjovsky2017wasserstein}.}

Let $\mathcal{X}$ be a compact set. Taking the sup-norm $\|f\|_\infty = \max_{x\in\mathcal{X}}\lvert f(x)\rvert$, the space $(\ctsbdd{\mathcal{X}},\lVert\cdot\rVert_\infty)$ is a normed vector space. We can then define the dual normed space $(\ctsbdd{\mathcal{X}}^*,\lVert\cdot\rVert)$, where we take
\begin{align*}
\ctsbdd{\mathcal{X}}^* &:= \{\phi\colon\ctsbdd{\mathcal{X}}\to\mathbb{R}\mid \phi\text{ is linear and continuous.}\} \\
\lVert\phi\rVert &:= \sup_{f\in \ctsbdd{\mathcal{X}}, \lVert f\rVert_\infty\le1}\  \lvert\phi(f)\rvert.
\end{align*}

Consider the mapping 
\begin{align*}
\Phi\colon(\probspace{\mathcal{X}},d_{\text{TV}})&\to(\ctsbdd{\mathcal{X}}^*,\lVert\cdot\rVert) \\
\Phi(\mu)(f) &:= \ev{x}{\mu}[f(x)].
\end{align*}

By linearity of expectation, this function indeed maps to the appropriate dual space and so is well-defined. Therefore, by the Riesz-Markov-Kakutani representation theorem \citep[Theorem 10]{kakutani1941concrete}, $\Phi$ is an isometric immersion. This allows us to regard convergence in TV distance and convergence with respect to $\lVert\cdot\rVert$ as essentially equivalent. This is unfortunate for the TV distance: convergence with respect to the latter norm in $\ctsbdd{\mathcal{X}}^*$ is regarded as `strong' convergence,\footnote{It is referred to as such in the standard literature. For example, see \citet[Definition 4.9-4]{kreyszig1978introductory}.} in effect limiting the capacity of a TV-based GAN to train towards a variety of real distributions.

By contrast, $\ctsbdd{\mathcal{X}}^*$ also comes equipped with a much \textit{weaker} topology.

\begin{weaktopology}[e.g., \citet{liu2017approximation}, Definitions 7-8]
Let $\mathcal{X}$ be a compact metric space. The \textbf{weak$^*$ topology} for $\probspace{\mathcal{X}}$ is the coarsest topology on $\probspace{\mathcal{X}}$ such that \[\{\mu\mapsto\mathbb{E}_\mu[f]\mid f\in\ctsbdd{\mathcal{X}}\}\] is a set of continuous linear functions on $\probspace{\mathcal{X}}$.

Moreover, we say that a sequence $(\mu_n)\subseteq\probspace{\mathcal{X}}$ \textbf{weakly converges} to a measure $\mu\in\probspace{\mathcal{X}}$ if, for all $f\in\ctsbdd{\mathcal{X}}$, \[
\mathbb{E}_{\mu_n}[f]\to\mathbb{E}_{\mu}[f] \text{ as }n\to\infty,
\] or equivalently, if $\mu_n\to\mu$ in the weak$^*$ topology.
\end{weaktopology}

If we can show that the Wasserstein distance captures the notion of weak$^*$ convergence, then we may claim that the WGAN gives us a more suitable choice of objective function than any other GAN. We can go further than this by providing a hierarchy of divergences with respect to their convergence.

\begin{strengths}[\citet{arjovsky2017wasserstein}, Theorem 2]\label{strengthcomp}
Let $\mathcal{X}$ be compact, and let $\mu, (\mu_n)\subseteq\probspace{\mathcal{X}}$. Then, considering all limits as $n\to\infty$,
\begin{enumerate}
\item $\dive{TV}{\mu_n}{\mu}\to0$ if and only if $\dive{JS}{\mu_n}{\mu}\to0$.
\item $\dive{W}{\mu_n}{\mu}\to0$ if and only if $\mu_n\xrightarrow{d}\mu$, where $\xrightarrow{d}$ denotes convergence in distribution.
\item If either $\dive{KL}{\mu}{\mu_n}\to0$ or $\dive{KL}{\mu_n}{\mu}\to0$, then $\dive{JS}{\mu_n}{\mu}\to0$.
\item If $\dive{JS}{\mu_n}{\mu}\to0$, then $\dive{W}{\mu_n}{\mu}\to0$.
\end{enumerate}
\end{strengths}
\begin{proof}
\hfill 
\begin{enumerate}
\item See Appendix A.

\item 
This comes from the standard result that $d_{\text{W}}$ gives a metric for the weak$^*$ topology of $(\ctsbdd{\mathcal{X}}, \lVert\cdot\rVert_\infty)$ on $\probspace{\mathcal{X}}$, and by definition, this is the topology of convergence in distribution. For a proof, see \citet[Theorem 6.9]{villani2008optimal}.

\item
By Pinsker's Inequality \citep[Section A.2, p. 371]{cesa2006prediction}, either case gives us one of
\begin{align*}
\dive{TV}{\mu_n}{\mu} &\le \sqrt{\frac{1}{2}\dive{KL}{\mu}{\mu_n}} \to0, \\
\dive{TV}{\mu_n}{\mu} &\le \sqrt{\frac{1}{2}\dive{KL}{\mu_n}{\mu}} \to0.
\end{align*}

\item
This comes from the fact, as argued above, that $d_{\text{TV}}$ and $d_{\text{W}}$ give the strong and weak$^*$ topologies on the dual of $(\ctsbdd{\mathcal{X}}, \lVert\cdot\rVert_\infty)$ when restricted to the space of probability measures on $\mathcal{X}$.

\end{enumerate}
\end{proof}

Finally, we observe how even a simple sequence of probability distributions converges under $d_{\text{W}}$ but not under $d_{\text{JS}}$ or $d_{\text{KL}}$. This serves as a witness to Theorem~\ref{strengthcomp}.

\begin{simplefailure}[\citet{arjovsky2017wasserstein}, Example 1]\label{simplefail}
Let $Z\sim U[0,1]$ be uniformly distributed over the unit interval. Let $\mathbb{P}_0$ be the distribution of $(0,Z)\in\mathbb{R}^2$. Now let $g_\theta(z) = (\theta,z)$, with $\theta\in\mathbb{R}$ and $\mathbb{P}_\theta$ the distribution for $g_\theta(Z)$. In this case:
\begin{itemize}
\item $\dive{W}{\mathbb{P}_0}{\mathbb{P}_\theta} = |\theta |$,
\item $\dive{JS}{\mathbb{P}_0}{\mathbb{P}_\theta} = \begin{cases}
\log2&    \text{if }\theta\neq0, \\
0&    \text{if }\theta=0,\end{cases}$
\item and $\dive{KL}{\mathbb{P}_\theta}{\mathbb{P}_0}=\dive{KL}{\mathbb{P}_0}{\mathbb{P}_\theta}=\begin{cases}
+\infty&    \text{if }\theta\neq0, \\
0&    \text{if }\theta=0.\end{cases}$
\end{itemize}
Hence, when $\theta_n\to0$, the sequence $(P_{\theta_n})_{n\in\mathbb{N}}$ converges to $\mathbb{P}_0$ only under the Wasserstein distance.
\end{simplefailure}

\subsection{On the Viability of WGAN}

In the above example, the JS divergence fails to give us continuous mapping $\theta\mapsto\mathbb{P}_{\theta}$, a desirable property. The next theorem shows us that, under mild assumptions, $\dive{W}{\mu}{\mu_{\theta}}$ is a continuous loss function on $\theta$.

\begin{wcts}[\citet{arjovsky2017wasserstein}, Theorem 1]\label{wasscts}
Let $\mathcal{X}$ be compact, and let $\mu\in\probspace{\mathcal{X}}$. Let $Z$ be a random variable over another space $\mathcal{Z}$. Let $g\colon \mathcal{Z}\times\mathbb{R}^d\to\mathcal{X}$ be a function, denoted $g_\theta(z)$. Let $\mu_\theta$ denote the distribution of $g_\theta(Z)$. Then,
\begin{enumerate}
\item If $g$ is continuous in $\theta$, so is $\dive{W}{\mu}{\mu_\theta}$,
\item If $g$ is locally Lipschitz and there are local Lipschitz constants $L(\theta,z)$ such that \[ \mathbb{E}_{z\sim\mu_\theta}[L(\theta,z)]<+\infty, \] then $\dive{W}{\mu}{\mu_\theta}$ is continuous everywhere, and differentiable almost everywhere.
\item Statements 1-2 are false for $d_{\text{JS}}$ and the two KL divergences.
\end{enumerate}
\end{wcts}
\begin{proof}
\hfill
\begin{enumerate}
\item
Let $\gamma$ be the distribution of $(g_\theta(Z), g_{\theta'}(Z))$, so that $\gamma\in\Pi(\mu_\theta, \mu_{\theta'})$. Then \eqref{emdist} gives us
\begin{align*}
\dive{W}{\mu_\theta}{\mu_{\theta'}} &\le \int_{\mathcal{X}\times\mathcal{X}}\lVert x-y \rVert \ \mathrm{d}\gamma \\
&= \mathbb{E}_{(x,y)\sim\gamma}[\lVert x-y\rVert] \\
&= \mathbb{E}_\mathcal{Z} [\lVert g_\theta(Z)-g_{\theta'}(Z) \rVert].
\end{align*}
Since $g$ is continuous in $\theta$, we have that $g_\theta(z)\xrightarrow{\theta\to\theta'}g_{\theta'}(z)$. Hence \[\lVert g_\theta-g_{\theta'} \rVert\to0\] point-wise in $z$. Since $\mathcal{X}$ is compact, there exists a positive constant $M$, independent of $\theta$ and $z$, such that for all $\theta, z$, we have $\lVert g_\theta(z)-g_{\theta'}(z)\rVert\le M$. By the Bounded Convergence Theorem, 
\begin{align*}
\lvert \dive{W}{\mu}{\mu_\theta} - \dive{W}{\mu}{\mu_{\theta'}} \rvert &\le \dive{W}{\mu_\theta}{\mu_{\theta'}} \\
&\le \mathbb{E}_\mathcal{Z} [\lVert g_\theta(z)-g_{\theta'}(z) \rVert] \\
&\xrightarrow{\theta\to\theta'} 0.
\end{align*}
The result follows.

\item
Take $g$ to be locally Lipschitz. This means that, fixing $(\theta, z)$, there exists a constant $L(\theta, z)$ and open set $U$ such that $(\theta, z)\in U$, and for every $(\theta', z')\in U$,
\begin{equation}\label{lipbound}
\lVert g_\theta(z)-g_{\theta'}(z') \rVert \le L(\theta, z)(\lVert \theta-\theta'\rVert + \lVert z-z'\rVert).
\end{equation}
Fix $z'=z$. Taking the expectation in \eqref{lipbound}, we have for all $(\theta', z)\in U$ that \[
\mathbb{E}_\mathcal{Z} [\lVert g_\theta(z)-g_{\theta'}(z) \rVert] \le \lVert \theta-\theta'\rVert\cdot\mathbb{E}_\mathcal{Z}[L(\theta, z)].
\]
Define $U_\theta:=\{\theta' \mid (\theta', z)\in U\}$. Since $U$ is open, $U_\theta$ is also open. Hence, by hypothesis, we may define $L(\theta) := \mathbb{E}_\mathcal{Z}[L(\theta, z)]$ and get for all $\theta'\in U_\theta$ that \[
\lvert \dive{W}{\mu}{\mu_\theta} - \dive{W}{\mu}{\mu_{\theta'}} \rvert \le \dive{W}{\mu_\theta}{\mu_{\theta'}} \le L(\theta)\cdot\lVert\theta-\theta'\rVert.
\]
Hence $\dive{W}{\mu}{\mu_\theta}$ is locally Lipschitz. Therefore,  $\dive{W}{\mu}{\mu_\theta}$ is everywhere continuous, and, by Rademacher's Theorem \citep[Theorem 3.1.6]{federer2014geometric}, differentiable almost everywhere.

\item Observe that Example \ref{simplefail} serves as the required counterexample.
\end{enumerate}
\end{proof}

Of course, in practice our generator functions will be neural nets. The following corollary shows that the previous result holds if we restrict our attention to these kinds of functions.

\begin{minemdist}[\citet{arjovsky2017wasserstein}, Corollary 1]
Let $g_\theta$ be any feedforward neural net\footnote{In other words, a function composed by affine transformations and pointwise nonlinearities which are smooth Lipschitz functions.} parametrised by $\theta$, and $\ndist$ a prior over $z$ such that $\ev{z}{\ndist}[\| z\|]<\infty$. Then the assumptions of Theorem~\ref{wasscts} are satisfied, and so $\dive{W}{\mu}{\mu_\theta}$ is continuous everywhere and differentiable almost everywhere.
\end{minemdist}
\begin{proof}
Since $g$ is a continuously differentiable function in $(\theta,z)$, for any fixed  $(\theta, z)$, we have for all $\varepsilon>0$ that $L(\theta, z)\le\lVert\nabla_{\theta,x}g_\theta(z)\rVert+\varepsilon$ is an acceptable local Lipschitz constant. It therefore remains to show that
\begin{equation}
\ev{z}{\ndist}[\lVert\nabla_{\theta,x}g_\theta(z)\rVert]<+\infty.
\end{equation}
This part of the proof is omitted. Refer to \citet[Appendix C]{arjovsky2017wasserstein} for the technical details.
\end{proof}

\subsection{The WGAN Procedure}

Evaluating $\dive{W}{\rdist}{p_\theta}$, where $p_\theta$ is the distribution of our generator $g_\theta$, is often intractable. A more tractable approach, justified by the Kantorovich-Rubinstein Duality, would be to solve the problem
\begin{equation}
\max_{w\in\mathcal{W}}\ \bigl(\mathbb{E}_{x\sim\rdist}[f_w(x)] - \mathbb{E}_{z\sim \ndist(z)}[f_w(g_\theta(z))]\bigr),
\end{equation}
where $\{f_w\}_{w\in\mathcal{W}}$ is a set of parametrised functions that are all $K$-Lipschitz for some $K$. If the supremum in \eqref{krd} is attained for some $w\in\mathcal{W}$, this process would yield a calculation of $\dive{W}{\rdist}{p_\theta}$ up to a multiplicative constant $K$. 

To minimise $\dive{W}{\rdist}{p_\theta}$ with respect to $\theta$, we could consider differentiating $\dive{W}{\rdist}{p_\theta}$ (up to a constant) by using back-propagation through equation \eqref{krd} via estimating $\mathbb{E}_{z\sim \ndist(z)}[\nabla_\theta f_w(g_\theta(z))]$. The following theorem shows that this process is principled under the assumptions of the previous results.

\begin{princproc}[\citet{arjovsky2017wasserstein}, Theorem 3]\label{thm:princproc}
Let $\mathcal{X}$ be compact, and let $\rdist\in\probspace{\mathcal{X}}$. Let $p_\theta$ be the distribution of $g_\theta(Z)$ with $Z$ a random variable with density $\ndist$ and $g_\theta$ a function satisfying the assumptions of Theorem~\ref{wasscts}. Then, there is a solution $f\colon\mathcal{X}\to\mathbb{R}$ to the problem \[
\max_{f\in\text{{\normalfont Lip}}_1(\mathcal{X})} \bigl(\mathbb{E}_{x\sim\rdist}[f(x)]-\mathbb{E}_{x\sim p_\theta}[f(x)]\bigr)
\] and we have \[
\nabla_\theta \dive{W}{\rdist}{p_\theta} = -\mathbb{E}_{z\sim \ndist(z)}[\nabla_\theta f(g_\theta(z))]
\] when both terms are well-defined.
\end{princproc}
\begin{proof}
The proof can be found in Appendix A.
\end{proof}

To roughly approximate finding the function $f$ that solves \eqref{krd}, \citet{arjovsky2017wasserstein} train a neural network parametrised with weights $w$ lying in a compact space $\mathcal{W}$, and then perform back-propagation through $\mathbb{E}_{z\sim \ndist(z)}[\nabla_\theta f_w(g_\theta(z))]$.\footnote{Refer to Algorithm 1 in  \citet{arjovsky2017wasserstein} for the precise formulation of the WGAN algorithm.}

Here, we refer to $f_w$ as our \textit{critic}, just as the original GAN had a \textit{discriminator} $D$. When $\mathcal{W}$ is compact,\footnote{This is enforced in the WGAN algorithm given by \citet{arjovsky2017wasserstein} by clipping the weights within a fixed hypercube, say $\mathcal{W} = [-0.01, 0.01]^l$, after each gradient update. See \citet{gulrajani2017improved} for an alternate approach to enforcing the Lipschitz constraint by adding a `gradient penalty' term to the discriminator loss function.} all the functions $f_w$ will be $K$-Lipschitz for some $K$ depending only on $\mathcal{W}$ and not the individual weights $w$, allowing us approximate \eqref{krd} up to an irrelevant scaling factor.

\subsection{Does WGAN Resolve Convergence Problems?}

\subsubsection*{Mode Collapse}

Arjovsky et al. recommend that WGAN critics be trained to optimality. In so doing, they claim, one avoids the phenomenon of \textbf{mode collapse}. The argument for this relies on the Goodfellow-Metz explanation for mode collapse. The reasoning is as follows: the hypothesis is that mode collapse comes from the fact (Proposition~\ref{optgenfact}) that the optimal generator for a given discriminator is a map to the points for which the discriminator assigns the highest probability, so that each generator update step in training the MM-GAN is a partial collapse towards this function. Therefore, if the discriminator is close enough to optimality, the data points to which it will assign the highest probability will be precisely those that arise from the real data distribution $\rdist$, and so an optimal generator function given this discriminator will be a delta function valued 1 on all such data points.

It has also been observed empirically that WGANs avoid mode collapse in cases where the MM-GANs do not \citep{arjovsky2017wasserstein, fedus2017many, lucic2017gans}.

\subsubsection*{Failure to Converge}

According to \citet{arjovsky2017wasserstein}, the fact that the EM distance is continuous and differentiable a.e. means that, unlike the MM-GAN, we can train the critic to optimality without worrying about the convergence of the generator distribution. 

Consider, for instance, the target and generator distributions given in Example~\ref{simplefail}. Unless the generator distribution has already matched the target distribution, the JS distance between the two is constantly $\log2$. This distance, as a result, gives no meaningful gradient for the generator to use for training, provided that the discriminator is sufficiently optimal so as to give an accurate estimate of the JS distance.\footnote{\citet[Theorem 2.4]{arjovsky2017towards} shows that this issue of vanishing gradients for the MM-GAN generator can be found to occur quite generally.} By contrast, the WGAN critic converges to a piecewise linear function through the constraint of its weights by clipping. In this sense, we see how the WGAN can tackle the issue of saturation.

\chapter{A Game-Theoretic Analysis of GANs}
\section{Many Paths to Equilibrium}
Throughout the previous chapter, I considered the training of GANs as a minimisation of an \textit{adversarial divergence}, itself the supremum over some loss function. In doing so, I took the perspective of GANs as optimising a generator, given a discriminator that is already optimal.

While such an approach does enjoy certain theoretical and empirical successes,\footnote{For instance, in Wasserstein GAN \citep{arjovsky2017wasserstein, arjovsky2017towards}.} there are also two potential drawbacks. 

Firstly, it may be computationally impractical to train a discriminator to optimality at each step. An approach that trains the generator and discriminator simultaneously, or trains the discriminator for $k$ iterations between each generator training step may simply converge to the solution more efficiently. When such an approach is taken with the neural net parameters trained via gradient descent, these approaches are referred to as \textbf{simultaneous} and \textbf{alternating gradient descent} (SimGD and AGD), respectively.

Secondly, as we have seen, certain GAN formulations do not perform well with optimal discriminators: the gradient along which the generator is being optimised collapses, meaning it converges to its optimum far more slowly. This was observed in an informal setting for MM-GANs in the original GAN paper by \citet{goodfellow2014generative}.

\subsection{Two Generalisations}

In this chapter, I will investigate the consequences of relaxing two constraints implicitly imposed by the \textit{adversarial divergence} view of training GANs.
\begin{enumerate}
\item We will now consider the cases in which the generator and discriminator are trained via simultaneous or alternating gradient descent. The former approach has been modelled in \citet{mescheder2017numerics}. The latter was first proposed by \citet{goodfellow2014generative} (Algorithm 1 in this dissertation), with the number $k$ of discriminator training steps between each generator training step treated as an algorithm hyperparameter to be carefully chosen.\footnote{For a comparison between AGD and training the discriminator to optimality on a toy example, refer to \citet{li2017towards}.}

\item We will allow for the discriminator and generator to be trained on objective functions whose absolute values differ. In particular, we now consider the discriminator and generator to seek to minimise \textbf{loss functions} $\loss{D}{}$ and $\loss{G}{}$, respectively. Hence, unless $\loss{D}{}=-\loss{G}{}$ as assumed in the previous chapters, we can no longer express the GAN objective by a single value function $V$.
\end{enumerate}

For a simple example of a GAN that cannot be expressed by a single value function, consider the following example.

\begin{nsgan}[\citet{goodfellow2014generative}]
Let $\mathcal{Z}\subseteq\eucl{\ell}$, $\mathcal{X}\subseteq\eucl{d}$ be ambient data spaces, let $\ndist$ be a prior distribution over $\mathcal{Z}$, and let $\rdist$ be the distribution of real data points over $\mathcal{X}$. The \textbf{non-saturating GAN} (NS-GAN) is the game specified by the minimisation of loss functions
\begin{align*}
\loss{D}{NS} &:= \ganobj{\log}{D}{G}, \\
\loss{G}{NS} &:= -\ev{z}{\ndist}[\log D(G(z))].
\end{align*}
where $G\colon\mathcal{Z}\to\mathcal{X}$, $D\colon\mathcal{X}\to[0,1]$.
\end{nsgan}

Such a game has been observed to enjoy empirical success: in particular, \citet{fedus2017many} show that on a variety of datasets, NS-GAN produces samples of a comparable quality to WGAN and its variants, while being easier to train.\footnote{Visual inspection is currently one of the most prominent methods of evaluation within the GAN literature. In general, finding a suitable quantitative evaluation model for GANs is one of the largest open problems in the field \citep[p. 42]{goodfellow2016nips}. Currently, the Inception Score \citep{salimans2016improved} and the Fr\'echet Inception Distance \citep{ramsauer2017two} are amongst the most popular performance metrics available.}

\section{A Brief Review of Game Theory}

Since the NS-GAN is not expressible by a single value function, our theoretical results about adversarial divergence minimisation no longer apply. To analyse GAN examples like NS-GAN, we require a game-theoretic framework. In doing so, we can provide answers to the following questions:
\begin{quote}
Do solutions to the GAN game exist? What is the nature of such solutions?

Are there training methods for the discriminator and generator that allow us to converge to such solutions?
\end{quote}

\subsection{Two-Player Games}

We begin with the definition of the type of game we are interested in. Unless stated otherwise, the definitions and results are adapted from \citet[Chapters 2-3]{osborne1994course}.

\begin{game}
A \textbf{strategic two-player game} $\langle (A_i), (u_i)\rangle_{i=1,2}$ consists of, for players $i=1,2$
\begin{itemize}
\item a nonempty set $A_i$ (the set of \textbf{actions} available to player $i$)
\item a \textbf{payoff or utility function} $u_i\colon A\to\mathbb{R}$, where $A = A_1\times A_2$.
\end{itemize}
If the set $A_i$ of actions of both players is finite, then the game is \textbf{finite}.
\end{game}

This game is referred to as \textit{strategic} as each player chooses their plan of action once and for all, with these choices being made simultaneously. Given such a game, the most commonly used solution concept is that of a \textbf{Nash equilibrium}. The Nash equilibrium captures the idea that each player holds the correct expectation about the other players' behaviour and acts rationally; in a Nash equilibrium, neither player can gain by deviating from their strategy.

\begin{nedef}
Let $a_{-i}$ denote the strategy of player 1 for $i=2$, and player 2 for $i=1$. A \textbf{Nash equilibrium of a strategic two-player game} $\langle (A_i), (u_i)\rangle_{i=1, 2}$ is a profile $a^*\in A$ of actions with the property that for $i=1, 2$ we have 
\[ u_i(a^*_{-i}, a^*_i) \ge u_i(a^*_{-i}, a_i) \text{ for all } a_i\in A_i.\]
\end{nedef}

\subsection{Minimax Games}

\begin{zerosum}
A strategic two-player game $\langle (A_i), (u_i)\rangle_{i=1,2}$ is \textbf{zero-sum} if $u_1=-u_2$.
\end{zerosum}

\begin{fanmmrem}
With an appropriate choice of \textit{minimax theorem},\footnote{See, e.g., \citet[Theorem 2]{fan1953minimax}.} we can show that an optimal solution to our adversarial divergence objective function exists, and that it coincides with a Nash equilibrium of a zero-sum game.

Unfortunately, such results do not apply to GANs. Typically, a minimax theorem requires the $A_i$ to be compact and the value function $V(\cdot, \cdot)$ to be convex in the first argument and concave in the second. In general, the value function will fail to have this property due to a lack of expressivity in the discriminator-generator function space. Even if we allow the discriminator and generator to range over a broader class of functions, the associated function spaces will no longer be compact. We see this failure in the mode collapse hypothesis, in which the minimax and maximin solutions are distinct.
\end{fanmmrem}

\subsection{Mixed Extensions to Games}

Given that the GAN game consists in practice of making updates to the parameter space rather than the function space, is there any way we can generalise the notion of a Nash equilibrium to recover the guarantee of the existence of such an equilibrium for the game? The answer is a qualified yes.

Our required generalisation is that of a \textit{mixed strategy Nash equilibrium}: this will be a steady state of the game in which the players' choices are not deterministic, but instead determined probabilistically according to some distribution.

In particular, we denote by $\Delta(A_i)$ the set of probability distributions over $A_i$ and refer to a member of $\Delta(A_i)$ as a \textbf{mixed strategy}\footnote{The players' mixed strategies are assumed to be independent.} of player $i$. A member of $A_i$ is referred to as a \textbf{pure strategy}. For any finite set $X$ and $\delta\in\Delta(X)$ we denote by $\delta(x)$ the probability that $\delta$ assigns to $x\in X$, and define the \textbf{support} of $\delta$ to be the set of elements $x\in X$ such that $\delta(x)>0$. Given a profile $(\alpha_1, \alpha_2)$ of mixed strategies, we have a probability distribution over the set $A$. If each $A_i$ is finite, then the probability of the action profile $a=(a_1, a_2)$ is $\alpha_1(a_1)\cdot\alpha_2(a_2)$, and player $i$'s evaluation of $(\alpha_1, \alpha_2)$ is $\sum_{a\in A} (\alpha_1(a_1)\cdot\alpha_2(a_2))\cdot u_i(a)$. This leads to the following definition.

\begin{mixgame}
The \textbf{mixed extension} of the strategic two-player game \hfill \break $\langle (A_i), (u_i)\rangle_{i=1,2}$ is the strategic two-player game $\langle (\Delta(A_i)), (U_i)\rangle_{i=1,2}$ in which $\Delta(A_i)$ is the set of probability distributions over $A_i$, and $U_i\colon \Delta(A_1)\times\Delta(A_2)\to\mathbb{R}$ assigns to each $\alpha\in\Delta(A_1)\times\Delta(A_2)$ the expected value under $u_i$ of the probability distribution over $A$ induced by $\alpha$, so that \[ U_i(\alpha) = \sum_{a\in A} (\alpha_1(a_1)\cdot\alpha_2(a_2))\cdot u_i(a)\]
\end{mixgame}

From this definition, it follows that $U_i$ is multilinear: for any mixed strategy profile $\alpha$, any mixed strategies $\beta_i$ and $\gamma_i$ of player $i$, and any number $\lambda\in[0,1]$, that \[ U_i(\alpha_{-i}, \lambda\beta_i +(1-\lambda)\gamma_i) = \lambda U_i(\alpha_{-i}, \beta_i) + (1-\lambda) U_i(\alpha_{-i}, \gamma_i). \] Moreover, when each $A_i$ is finite, we have \[ U_i(\alpha) = \sum_{a_i\in A_i} \alpha_i(a_i)U_i(\alpha_{-i}, e(a_i)) \] for any mixed strategy profile $\alpha$, where $e(a_i)$ is the degenerate mixed strategy of player $i$ corresponding to the pure strategy of just choosing $a_i$.

\begin{mixnash}
A \textbf{mixed strategy Nash equilibrium of a strategic two-player game} is a Nash equilibrium of its mixed extension.
\end{mixnash}

It is plain to see, since each $U_i$ is multilinear, and any mixed strategy exclusively employing degenerate distributions can be readily identified with a strategy in $A$, that the set of Nash equilibria of a strategic game is a subset of its set of mixed strategy Nash equilibria. 

In a well-known result, \citet{nash1950equilibrium} showed that a mixed strategy Nash equilibrium is guaranteed to exist if the action spaces $A_1, A_2$ are finite. Of course, this condition fails to hold for the generator and discriminator of a GAN, whose strategy spaces are parameter spaces in $\mathbb{R}^n, \mathbb{R}^m$, respectively. However, we may restrict the parameter spaces such that the GAN game is a \textbf{continuous game}, where $A_1, A_2$ are non-empty compact metric spaces and $u_1, u_2$ are continuous functions on $A$. We then get a result that applies to the GAN game.

\begin{glicksberg}[Glicksberg's Theorem \citep{glicksberg1952further}]\label{glickthm}
Let $G$ be a continuous game, in the sense described above. Then $G$ has a mixed strategy Nash equilibrium.
\end{glicksberg}
\begin{proof}[Proof Idea]
Though Glicksberg's own proof relies on a generalisation of the fixed-point theorem used to prove Nash's existence theorem, an alternative proof given in \citet[Lecture 6]{ozdaglar2010game} makes use of Nash's existence theorem without any further fixed-point result. In particular, the continuous game is approximated with a sequence of finite games, each corresponding to successively finer discretisations of the original game. Nash's existence theorem produces an equilibrium for each approximation, which we can show using the weak topology and the continuity assumptions to converge to an equilibrium of the original game. The full proof is given in Section~\ref{sec:glickpf}.
\end{proof}

\section{Game-Theoretic Results for GANs}

\subsection{Mixing GANs}

Two immediate problems arise with applying Glicksberg's Theorem to GANs. Firstly, we may for practical purposes be concerned about the \textit{support size} of the mixed strategy. In other words: how many generators does it take to fool a discriminator of a certain strength? A classical result tells us that, so long as our generators are capable of producing simple Gaussian distributions, we can arbitrarily approximate $\rdist$. Yet if the support size necessary to achieve this is too large, then training so many generators will be computationally prohibitive.

Secondly, it is not clear what it \textit{means} for a generator to employ a mixed strategy. Of course, the generator induces a probability distribution $\gdist$ which it aims to be close to $\rdist$. However, this is not the relevant probability distribution when considering mixed strategies in the GAN game. Instead, a mixed strategy would be a distribution over the \textit{parameters} of $G$ - in other words, a mixed strategy is a probability distribution over functions that induce probability distributions!

\citet{arora2017generalization} show how a mixture of finitely many generators and discriminators may approximate the minimax solution of the GAN game. Firstly, we define such an approximate equilibrium.

\begin{epsiloneq}\label{def:epsiloneq}
Let $\langle (\Delta(A_i), (U_i)\rangle_{i=1,2}$ be a mixed strategic two-player game, and let $\varepsilon>0$. A mixed strategy profile $\alpha^*=(\alpha^*_1,\alpha^*_2)$ is an \textbf{$\varepsilon$-approximate equilibrium for $G$} if, for some $i=1,2$,
\begin{equation}
U_i(\alpha^*)\ge U_i(\alpha'_i, \alpha^*_{-i})-\varepsilon\text{ for all }\alpha'_i\in\Delta(A_i).
\end{equation}
\end{epsiloneq}

This leads Arora et al. to prove a theorem showing both the existence of an $\varepsilon$-approximate equilibrium using a finite mixture of generators and single discriminator, and a procedure for constructing an $\varepsilon$-approximate \textit{pure} equilibrium. In doing so, we address the two concerns above. This theorem holds for the \textbf{neural net divergence}, defined here.

\begin{nndist}[\citet{arora2017generalization}]
Let $\mathcal{X}\subseteq\eucl{d}$ be an ambient data space, let $\gdist$ be the distribution induced by our generator function, and let $\rdist$ be the distribution of real data points over $\mathcal{X}$. Let $\mathcal{F}$ be a class of neural networks from $\mathbb{R}^d$ to $[0,1]$, such that if $f\in\mathcal{F}$, then $1-f\in\mathcal{F}$. Let $\phi$ be a measuring function: that is, $\phi\colon [0,1]\to\mathbb{R}$ be concave and monotone. The \textbf{neural $\mathcal{F}$-divergence} with respect to $\phi$ between two distributions $\mu, \nu$ supported on $\mathcal{X}$ is defined as
\begin{equation}
d_{\mathcal{F}, \phi}(\mu \| \nu) := \sup_{D\in\mathcal{F}} \bigl(-J^D_{\phi}(D, \nu)\bigr),
\end{equation}
where
\begin{equation}
J^D_{\phi}(D, \nu) = -\biggl(\ev{x}{\mu}[\phi(D(x))] + \ev{y}{\nu}[\phi(1-D(y))]\biggr)
\end{equation}
\end{nndist}

\begin{nnrem}
Note that, in seeking a generator $G$ that minimises \[d_{\mathcal{F}, \phi}(\rdist \| \gdist),\] we are seeking a solution to a zero-sum game, with the loss functions defined as expected. 

Furthermore, the neural net divergence may be seen as a generalisation of the practical instantiations of the JS and Wasserstein distance (setting $\phi(x) = \log(x)$ or $\phi(x) = x,$ respectively), where the term `practical' denotes the fact that our function space consists of neural networks. As such, any results we can prove about minimising neural net divergence will be highly relevant to the GAN examples most prominently considered so far.
\end{nnrem}

We are now in a position to give the main positive result with regards to GANs and the existence of equilibria.

\begin{arorapure}[\citet{arora2017generalization}, Theorem 4.3]\label{arorapureeq}
Let $\phi$ be an $L_\phi$-Lipschitz concave measuring function bounded in $[-\Delta,\Delta]$, and suppose the generator and discriminators are $L$-Lipschitz with respect to the parameters and $L'$-Lipschitz with respect to inputs. Furthermore, suppose the generator can approximate any point mass: that is, for all points $x$ and any $\epsilon>0$, there is a generator such that $\ev{z}{\ndist}[\lVert G(z)-x \rVert]\le\epsilon$. 

Suppose the generator and discriminator are both $k$-layer neural networks ($k\ge2$) with $p$ parameters, and the last layer uses the ReLU activation function $f(x) = \max\{0,x\}$. Then there exists $(k+1)$-layer neural networks of generators $G$ and discriminator $D$ with $O\left(\frac{\Delta^2p^2\log(LL'L_\phi\cdot p/\varepsilon)}{\varepsilon^2}\right)$ parameters, such that there exists an $\varepsilon$-approximate pure equilibrium with value $2\phi(1/2)$ to the game induced by $d_{\mathcal{F}, \phi}(\rdist \| \gdist)$.
\end{arorapure}
\begin{proof}
The proof can be found in Appendix B, Section~\ref{sec:arorapf}.
\end{proof}

We have now seen how, when we consider GANs \textit{qua} strategic zero-sum two-player games, we are able to establish the existence of approximate solutions to a broad variety of GAN games even when taking into account the limitations of the generator and discriminator functions.

Of course, non-zero-sum games such as the NS-GAN are not covered by this result, which only applies when the generator and discriminator are optimising for the same value function. However, as shown previously, Glicksberg's Theorem guarantees the existence of mixed strategy Nash equilibria for such a game, given the plausible assumption that the parameter spaces of the neural network are compact. It is an open problem whether the procedure for producing a single neural network that simulates a `mixed strategy' can be extended to guarantee the existence of pure equilibria in non-zero-sum games, or whether we could get a realistic bound for the support size of such a mixture in exchange for settling for an $\varepsilon$-approximate equilibrium.

\section{Frontiers of Research}

In this final section, I will briefly review some of the ongoing research in game theory that is applicable to better understanding the dynamics and possible solutions of GAN games.

\subsection{Existence does not Guarantee Convergence}
As \citet[p. 15]{arora2017generalization} observe, in the practical cases in which $V(\cdot, \cdot)$ is not convex-concave, the mere existence of an equilibrium does not guarantee that a simple algorithm like gradient descent will converge to it. The following is a pathological example.\footnote{For further examples, refer to \citet[Appendix C]{arora2017generalization} Moreover, \citet{mertikopoulos2018cycles} proves the existence of cycling behaviour for two players adopting another approach for finding the solution to a zero-sum game.} for a non-GAN zero-sum game that fails to converge to its equilibrium using gradient descent.

\begin{cycles}[\citet{goodfellow2016nips}, Sections 7.2 and 8.2]\label{ex:cycles}
Consider a zero-sum game with two players that each control a single scalar value. The minimising player controls scalar $x$, the maximising player controls scalar $y$, and the value function for the game is 
\begin{equation}
V(x,y) = xy.
\end{equation}
By solving $\partial_x V(x,y)=0$ and $\partial_y V(x,y)=0$, we can establish that there is a saddle point at $x=y=0$. Moreover, this saddle point is a Nash equilibrium: if the minimising player fixes $x=0$, the maximising player cannot attain a better score than at $y=0$, and vice versa.

Now, suppose the players were to learn this equilibrium via SimGD. To simplify the problem, we imagine that gradient descent is a continuous time process with an infinitesimal learning rate, so that the SimGD is described by the following system of partial differential equations:
\begin{align}
\frac{\partial x}{\partial t} = -&\frac{\partial}{\partial x} V(x(t),y(t)), \\
\frac{\partial y}{\partial t} = &\frac{\partial}{\partial y} V(x(t),y(t)).
\end{align}
Clearly, these evaluate to 
\begin{align}
\frac{\partial x}{\partial t} &= -y(t), \\
\frac{\partial y}{\partial t} &= x(t).
\end{align}
Differentiating (4.4.5) yields
\begin{equation}
\frac{\partial^2y}{\partial t^2} = \frac{\partial x}{\partial t} = -y(t).
\end{equation}
This differential equation has the solution
\begin{align}
x(t) &= x(0)\cos(t) -y(0)\sin(t) \\
y(t) &= x(0)\sin(t) +y(0)\cos(t).
\end{align}

These dynamics form a circular orbit, so that SimGD with an infinitesimal learning rate will orbit around the equilibrium forever, at the same radius that it was initialised. Moreover, a larger learning rate can overshoot, causing SimGD to spiral outward forever. Hence, unless the players find themselves at the equilibrium \textit{upon initialisation}, they will never approach the equilibrium using SimGD.
\end{cycles}

\subsection{Convergence to Other Types of Equilibria}

Some recent work has established the existence of other types of equilibria that a GAN will be guaranteed to converge to, or other types of GAN that are guaranteed to converge to an equilibrium.

As an example of the former, \citet{hazan2017efficient} defines a natural notion of \textit{regret} for the players of a GAN game, and gives gradient-based methods that guarantee convergence to a newly-defined notion of local, regret-based equilibrium. Though the paper specifically refers to MM-GAN games, the result can in fact be generalised to any GAN game whose loss functions are bounded, Lipschitz, and have Lipschitz gradients.

It remains to be seen, however, whether a generator-discriminator pair that attains such an equilibrium produces samples that synthesise $\rdist$ well. Could there be a relation that holds between $\rdist$ and $\gdist$ when an equilibrium is reached? Recall, in the case of the IMM-GAN, the minimax solution coincided with the two distributions in fact being equal.

As an example of the latter, \citet{unterthiner2017coulomb} introduces a new version of the GAN problem which treats the data samples as charged particles on a potential field. This GAN model possesses only one Nash equilibrium, which is optimal in that we get $\rdist=\gdist$.

However, this approach suffers from the drawback that it requires the discriminator to learn slow enough to accurately estimate the potential function induced by the generator, and that the generator must in turn learn even more slowly.\footnote{\citet[p. 7]{unterthiner2017coulomb}.} Moreover, the samples it produces for some standard datasets are not particularly visually appealing (see Figure~\ref{fig:coulomb_samples}).

\begin{figure}
	\includegraphics[width=\linewidth]{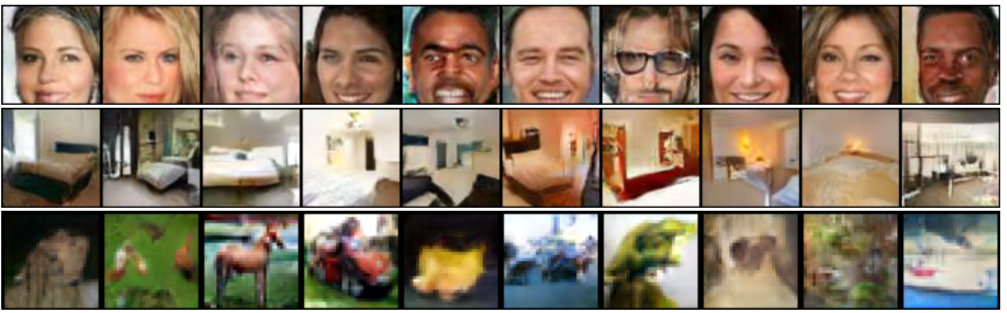}
	\caption{Images sampled from a Coulomb GAN after training on the CelebA (first row), LSUN bedrooms (second row), and CIFAR 10 (last row) datasets. Image taken from \citet{unterthiner2017coulomb}.}
	\label{fig:coulomb_samples}
\end{figure}

\chapter{Conclusion}
The literature on GANs is still in a very early stage, with the vast majority of papers developing their theory having been published within the last 18 months. As we have seen, there is still some way to go in understanding GAN training: what it means for GANs to perform well, and what guarantees there are on whether a GAN will perform well.

Nonetheless, both the topological and game-theoretic perspectives on GANs allow for us to propose versions of GANs that avoid the immediate pitfalls like mode collapse. Moreover, they allow us to make use of the theoretical insights from these areas: for instance, the Wasserstein GAN is inspired by previous research in optimal transport.

There is still plenty of work to be done on GANs: the game-theoretic approach seems particularly promising and under-explored. It seems, moreover, that any progress made with training GANs using this approach would require tools in algorithmic game theory that would have applications elsewhere in the field.

\appendix
\chapter{Omitted Proofs from Chapter 3}
\begin{proof}[Proof of Theorem~\ref{strengthcomp}, Part 1]
\hfill
\begin{itemize}
\item[($\Rightarrow$)]
Let $\mu_m$ be the mixture distribution $\mu_m = \frac{1}{2}(\mu_n+\mu)$, so $\mu_m$ depends on $n$. For any signed measure $\nu$, define $\tvnorm{\nu} := \sup_{A\subseteq\mathcal{X}}\lvert\nu(A)\rvert$ for all Borel sets $A$. Then
\begin{align*}
\dive{TV}{\mu_m}{\mu_n} &= \tvnorm{\mu_m-\mu_n} \\
&= \frac{1}{2}\tvnorm{\mu-\mu_n} \\
&= \frac{1}{2}\dive{TV}{\mu_n}{\mu}\le\dive{TV}{\mu_n}{\mu}.
\end{align*}
Hence, $\dive{TV}{\mu_m}{\mu_n}$ tends to 0 if $\dive{TV}{\mu_n}{\mu}$ tends to 0.

Now let $f_n = \frac{\mathrm{d}\mu_n}{\mathrm{d}\mu_m}$ be the Radon-Nykodim derivative between $\mu_n$ and $\mu_m$. By construction, we have for every Borel set $A$ that $\mu_n(A)\le2\mu_m(A)$. Picking $A=\{f_n>3\}$, we get \[
\mu_n(A) = \intoverofwrt{A}{f_n}{\mu_m}\ge3\mu_m(A),
\]
and so by these two inequalities, $\mu_m(A)=0$. Hence $f_n\le3$ almost everywhere, with respect to the measures $\mu_m, \mu_n$, and $\mu$.

Now fix $\varepsilon>0$, and $A_n = \{f_n>1+\varepsilon\}$. Then \[
\mu_n(A_n) = \intoverofwrt{A_n}{f_n}{\mu_m}\ge(1+\varepsilon)\mu_m(A_n).
\]

Hence
\begin{align*}
\varepsilon\mu_m(A_n) &\le \mu_n(A_n)-\mu_m(A_n) \\
&\le \lvert \mu_n(A_n)-\mu_m(A_n)\rvert \\
&\le \dive{TV}{\mu_n}{\mu_m} \\
&\le \dive{TV}{\mu_n}{\mu}.
\end{align*}

Furthermore,
\begin{align*}
\mu_n(A_n) &\le \mu_m(A_n)+\lvert\mu_n(A_n)-\mu_m(A_n)\rvert \\
&\le \frac{1}{\varepsilon}\dive{TV}{\mu_n}{\mu} +\dive{TV}{\mu_n}{\mu_m} \\
&\le \left(\frac{1}{\varepsilon}+1\right)\dive{TV}{\mu_n}{\mu}.
\end{align*}

Therefore,
\begin{align*}
\dive{KL}{\mu_n}{\mu_m} &= \intoverofwrt{\mathcal{X}}{\log(f_n)}{\mu_n} \\
&\le\log(1+\varepsilon)+\intoverofwrt{A_n}{\log(f_n)}{\mu_n} \\
&\le\log(1+\varepsilon)+\log(3)\mu_n(A_n) \\
&\le\log(1+\varepsilon)+\log(3)\left(\frac{1}{\varepsilon}+1\right)\dive{TV}{\mu_n}{\mu}.
\end{align*}

By taking $\limsup$ on both sides, we get that \[0\le\limsup\dive{KL}{\mu_n}{\mu_m}\le\log(1+\varepsilon)\] for all $\varepsilon>0$, and so $\dive{KL}{\mu_n}{\mu_m}\to0$.

Likewise, we define $g_n = \frac{\mathrm{d}\mu}{\mathrm{d}\mu_m}$ and $B_n = \{g_n>1+\varepsilon\}$, showing \textit{mutatis mutandis} that $\dive{KL}{\mu}{\mu_m}\to0$. From this, we conclude that \[
\dive{JS}{\mu_n}{\mu} = \frac{1}{2}\bigl(\dive{KL}{\mu_n}{\mu_m} + \dive{KL}{\mu}{\mu_m}\bigr) \to0.
\]

\item[($\Leftarrow$)] Using the triangle and Pinsker's inequalities, we get
\begin{align*}
\dive{TV}{\mu_n}{\mu} &\le \dive{TV}{\mu_n}{\mu_m}+\dive{TV}{\mu}{\mu_m} \\
&\le \sqrt{\frac{1}{2}\dive{KL}{\mu_n}{\mu_m}} + \sqrt{\frac{1}{2}\dive{KL}{\mu}{\mu_m}} \\
&\le 2\sqrt{\frac{1}{2}\dive{JS}{\mu_n}{\mu}} \to0.
\end{align*}

\end{itemize}

\end{proof}

\begin{proof}[Proof of Theorem~\ref{thm:princproc}]
Define
\begin{align*}
V(\tilde{f},\theta) &:= \ev{x}{\rdist}[\tilde{f}(x)]-\ev{x}{g_\theta}[\tilde{f}(x)] \\
&= \ev{x}{\rdist}[\tilde{f}(x)]-\ev{z}{\ndist}[\tilde{f}(g_\theta(x))],
\end{align*}
where $\tilde{f}\in\mathcal{F} = \text{{\normalfont Lip}}_1(\mathcal{X})$ and $\theta\in\mathbb{R}^d$.

Since $\mathcal{X}$ is compact, we know by the Kantorovich-Rubinstein duality (Theorem~\ref{krdualthm}) that there is an $f\in\mathcal{F}$ such that
\begin{equation}\label{eq:princproc}
\dive{W}{\rdist}{g_\theta} = \sup_{\tilde{f}\in\mathcal{F}} V(\tilde{f}, \theta) = V(f,\theta).
\end{equation}

Now define $X^*(\theta)$ to be the set of $f$ such that \eqref{eq:princproc} holds. The duality tells us that this set is non-empty. Moreover, by an envelope theorem \citep[Theorem 1]{milgrom2002envelope} we have
\begin{equation}
\nabla_\theta \dive{W}{\rdist}{g_\theta} = \nabla_\theta V(f,\theta)
\end{equation}
for any $f\in X^*(\theta)$, when both terms are well-defined.

Now take an $f\in X^*(\theta)$. Then, when the first two terms are well-defined,
\begin{align*}
\nabla_\theta \dive{W}{\rdist}{g_\theta} &= \nabla_\theta V(f,\theta) \\
&= \nabla_\theta [\ev{x}{\rdist}[f(x)]-\ev{z}{\ndist}[f(g_\theta(x))]] \\
&= -\nabla_\theta \ev{z}{\ndist}[f(g_\theta(x))].
\end{align*}
The technical remainder of this proof is to show that 
\begin{equation}\label{eq:princproctech}
-\nabla_\theta \ev{z}{\ndist}[f(g_\theta(x))] = -\ev{z}{\ndist}[\nabla_\theta f(g_\theta(x))].
\end{equation}

Since $f\in\mathcal{F}$, $f$ is 1-Lipschitz. Moreover, $f(g_\theta(z))$ is locally Lipschitz on $(\theta, z)$ with constants $L(\theta, z)$ given by the assumption on $g$. Hence, by Rademacher's Theorem \citep[Theorem 3.1.6]{federer2014geometric}, $f(g_\theta(z))$ is differentiable almost everywhere for $(\theta, z)$ jointly. In other words, the set $A = \{(\theta, z)\mid f\circ g\text{ is not differentiable}\}$ has measure $0$.

By Fubini's Theorem, this implies that for almost every $\theta$, the section $A_\theta = \{z\mid (\theta, z)\in A\}$ has measure $0$. Fix some $\theta_0$ such that the measure of $A_{\theta_0}$ is null, and the RHS of equation~\eqref{eq:princproctech} is well-defined. For this $\theta_0$, we have that $\nabla_\theta f(g_\theta(z))|_{\theta_0}$ is well-defined for almost any $z$, and $\ndist$-almost everywhere.

By our Lipschitz assumption, we know that 
\begin{equation}
\ev{z}{\ndist}[\lVert\nabla_\theta f(g_\theta(z))|_{\theta_0}\rVert]\le\ev{z}{\ndist}[L(\theta_0, z)]<+\infty,
\end{equation}
so $\ev{z}{\ndist}[\nabla_\theta f(g_\theta(z))|_{\theta_0}]$ is well-defined for almost every $\theta_0$. Therefore,
\begin{multline}\label{eq:princproc2}
\frac{\ev{z}{\ndist}[f(g_\theta(z))]-\ev{z}{\ndist}[f(g_{\theta_0}(z))]-\bigl\langle(\theta-\theta_0,\ev{z}{\ndist}[\nabla_\theta f(g_\theta(z))|_{\theta_0}]\bigr\rangle}{\lVert\theta-\theta_0\rVert} \\
= \ev{z}{\ndist}\left[\frac{f(g_\theta(z))-f(g_{\theta_0}(z))-\bigl\langle(\theta-\theta_0,\nabla_\theta f(g_\theta(z))|_{\theta_0}\bigr\rangle}{\lVert\theta-\theta_0\rVert}\right].
\end{multline}

By the differentiability of $f\circ g$, the term inside the integral converges $\ndist$-almost everywhere to 0 as $\theta\to\theta_0$. Moreover,
\begin{multline}
\biggl\lVert\frac{f(g_\theta(z))-f(g_{\theta_0}(z))-\bigl\langle(\theta-\theta_0,\nabla_\theta f(g_\theta(z))|_{\theta_0}\bigr\rangle}{\lVert\theta-\theta_0\rVert}\biggr\rVert \\
\le \frac{\lVert\theta-\theta_0\rVert L(\theta_0, z)+\lVert\theta-\theta_0\rVert \cdot \lVert\nabla_\theta f(g_\theta(z))|_{\theta_0}\rVert}{\lVert\theta-\theta_0\rVert} \\
\le 2L(\theta_0,z).
\end{multline}
Furthermore, since $\ev{z}{\ndist}[2L(\theta_0,z)]<+\infty$ by the Lipschitz  assumption, by dominated convergence equation~\ref{eq:princproc2} tends to 0 as $\theta\to\theta_0$, so that
\begin{equation}
\nabla_\theta \ev{z}{\ndist}[f(g_\theta(x))] = \ev{z}{\ndist}[\nabla_\theta f(g_\theta(x))]
\end{equation}
for almost every $\theta$, and in particular when the RHS is well-defined. The LHS is also proven to exist simultaneously.
\end{proof}

\chapter{Omitted Proofs from Chapter 4}
\section{Proof of Theorem~\ref{glickthm} (Glicksberg's Theorem)}\label{sec:glickpf}

This proof is as in \citet[Lecture 6]{ozdaglar2010game}. We first give a generalised definition of a continuous game.

\begin{ctsgame}
A \textbf{continuous game} is a game $\langle\mathcal{I}, (S_i),(u_i)\rangle$ where $\mathcal{I}$ is a finite set, the $S_i$ are non-empty compact metric spaces, and the $u_i\colon S\to\mathbb{R}$ are continuous functions valued on $S = \times_{i\in\mathcal{I}} S_i$.
\end{ctsgame}

Let $u = (u_1,...,u_I)$ and $\tilde{u} = (\tilde{u}_1,...,\tilde{u}_I)$ be two profiles of utility functions ($\lvert\mathcal{I}\rvert=I$) defined on $S$ such that, for each $i\in\mathcal{I}$, the functions $u_i, \tilde{u}_i$ are bounded and measurable. We can define the distance between the utility function profiles $u$ and $\tilde{u}$ as
\begin{equation}
\max_{i\in\mathcal{I}}\sup_{s\in S}\lvert u_i(s)-\tilde{u}_i(s)\rvert.
\end{equation}

Clearly, this distance is symmetric, obeys the triangle inequality, and is positive definite. Let $G = \langle\mathcal{I}, (S_i),(u_i)\rangle$ and $\tilde{G} = \langle\mathcal{I}, (S_i),(\tilde{u}_i)\rangle $ be two games, and let $\sigma$ be an equilibrium of $G$. We can show that, if $u$ and $\tilde{u}$ are close with respect to the distance given above, $\sigma$ is an $\varepsilon$-equilibrium of $\tilde{G}$. Here, the definition of a $\varepsilon$-equilibrium for $G$ generalises the definition given for a two-player game (Definition~\ref{def:epsiloneq}).

\begin{glickprop1}\label{prop:glick1}
Let $G$ be a continuous game. Assume that $\sigma^k\to\sigma$ and $\varepsilon^k\to\varepsilon$ as $k\to\infty$, where for each $k$, $\sigma^k$ is an $\varepsilon^k$-equilibrium of $G$. Then $\sigma$ is an $\varepsilon$-equilibrium of $G$.
\end{glickprop1}
\begin{proof}
We have by definition
\begin{equation}\label{eq:prop1}
u_i(s_i,\sigma^k_{-i})\le u_i(\sigma^k) +\varepsilon^k \quad \forall i\in\mathcal{I}, \ \forall s_i\in S_i.
\end{equation}
Taking the limit as $k\to\infty$ in \eqref{eq:prop1}, and using the continuity of the utility functions together with the convergence of the mixture distributions under the weak topology, we obtain
\begin{equation}
u_i(s_i,\sigma_{-i})\le u_i(\sigma)+\varepsilon \quad \forall i\in\mathcal{I}, \ \forall s_i\in S_i.
\end{equation}
Hence $\sigma$ is an $\varepsilon$-equilibrium of $G$.
\end{proof}

The next step is to define a notion of closeness of two strategic games, given the distance of utility function.

\begin{glickdef1}
Let $G = \langle\mathcal{I}, (S_i),(u_i)\rangle$ and $G' = \langle\mathcal{I}, (S_i),(u'_i)\rangle$ be two strategic games. We say that $G'$ is an $\alpha$-\textbf{approximation to} $G$ if for all $i\in\mathcal{I}$ and $s\in S$, we have
\begin{equation}
\lvert u_i(s) - u'_i(s)\rvert\le\alpha.
\end{equation}
\end{glickdef1}

\begin{glickprop2}\label{prop:glick2}
If $G'$ is an $\alpha$-approximation to $G$ and $\sigma$ is an $\varepsilon$-equilibrium of $G'$, then $\sigma$ is an $(\varepsilon+2\alpha)$-equilibrium of $G$.
\end{glickprop2}
\begin{proof}
For all $i\in\mathcal{I}$ and $s_i\in S_i$, we have
\begin{align*}
u_i(s_i,\sigma_{-i})-u_i(\sigma) &= \bigl(u_i(s_i,\sigma_{-i})-u'_i(s_i,\sigma_{-i})\bigr) \\
&\quad +\bigl(u'_i(s_i,\sigma_{-i}) -u'_i(\sigma)\bigr) \\
&\quad +\bigl(u'_i(\sigma)-u_i(\sigma)\bigr) \\
&\le  \alpha+\varepsilon+\alpha = \varepsilon+2\alpha.
\end{align*}
\end{proof}

The next proposition gives us the `discretisation' necessary to approximate our continuous game to an arbitrary degree of accuracy.

\begin{glickprop3}\label{prop:glick3}
For any continuous game $G$ and any $\alpha>0$, there exists an `essentially finite' game which is an $\alpha$-approximation to $G$.
\end{glickprop3}
\begin{proof}
Since $S$ is a compact metric space with metric $d$, the utility functions $u_i$ are uniformly continuous. That is, for all $\alpha>0$, there exists some $\varepsilon>0$ such that whenever $d(s,t)\le\varepsilon$,
\begin{equation}
\lvert u_i(s)-u_i(t)\rvert\le\alpha.
\end{equation}
Moreover, since $S_i$ is compact, it can be covered with finitely many open balls $U_i^j$, each with radius less than $\varepsilon$. We assume without loss of generality that these balls are disjoint and non-empty.

Now pick some $s_i^j\in U_i^j$ for each $i, j$. Then we define our `essentially finite' game $G'$ with utility functions $u'_i$ to be given by
\begin{equation}
u'_i(s) = u_i(s_1^j,...,s_I^j), \quad \forall s\in U^j = \times_{k=1}^I U_k^j.
\end{equation}
This game is `essentially finite', in that while the utility functions are valued on all of $S$, they nonetheless attain only finitely many values.

Then, for all $s\in S$ and all $i\in\mathcal{I}$, we have by uniform continuity that
\begin{equation}
\lvert u'_i(s)-u_i(s)\rvert\le\alpha,
\end{equation}
since $d(s, s^j)\le\varepsilon$ for all $j$. This gives us the desired result.
\end{proof}

We can now prove Glicksberg's Theorem.

\begin{proof}[Proof of Theorem~\ref{glickthm}]
Let $(\alpha^k)$ be a sequence such that $\alpha^k\to0$ as $k\to\infty$. By Lemma~\ref{prop:glick3}, there exists for each $\alpha^k$ an $\alpha^k$-approximation $G^k$ of $G$. 

Since each $G^k$ is `essentially finite' for each $k$, we can use Nash's existence theorem to guarantee the existence of a Nash equilibrium, equivalently a 0-approximate Nash equilibrium. Denote this by $\sigma^k$. Then, by Lemma~\ref{prop:glick2}, $\sigma^k$ is a $2\alpha^k$-equilibrium of $G$. 

Since $S$ is compact, the space of mixed distributions $\Delta(S)$ is compact, so $\{\sigma^k\}$ has a convergent subsequence. Without loss of generality, assume that $\sigma^k\to\sigma$. Since $2\alpha^k\to0$ and $\sigma^k\to\sigma$ as $k\to\infty$, it follows by Lemma~\ref{prop:glick1} that $\sigma$ is a 0-approximate equilibrium, hence a Nash equilibrium, of $G$.
\end{proof}

\section{Proof of Theorem~\ref{arorapureeq}}\label{sec:arorapf}

We must first show that there is a finite mixture of generators and discriminators that can approximate the equilibrium that exists for infinite mixtures. 

Suppose $\phi$ is an $L_{\phi}$-Lipschitz concave measuring function bounded in $[-\Delta,\Delta]$. Let $\mathcal{U}$, $\mathcal{V}\subseteq\mathbb{R}^p$ be the (compact) parameter spaces of the generator and discriminator, respectively. Suppose the generator and discriminators are $L$-Lipschitz with respect to the parameters and $L'$-Lipschitz with respect to inputs. Suppose further the value function $F$ for the minimax game is given by the neural $\mathcal{F}$-divergence\footnote{Here, $\mathcal{F} = \{D_v\mid v\in\mathcal{V}\}$.} with respect to $\phi$:
\begin{equation}
F(u, v) = \ev{x}{\rdist}[\phi(D_v(x))]+\ev{z}{\ndist}[\phi(1-D_v(G_u(z))].
\end{equation}

Furthermore, we suppose the generator can approximate any point mass: that is, for all points $x$ and any $\varepsilon>0$, there is a generator such that $\ev{z}{\ndist}[\lVert G(z)-x \rVert]\le\varepsilon$.

Note: all of the proofs within this section are as in \citet[Appendix B]{arora2017generalization}.

\begin{aroralemma1}[\citet{arora2017generalization}, Theorem 4.2]\label{al1}
In the above setting, there exists a universal constant $C>0$ such that for any $\varepsilon>0$, there exists \[ T = \frac{C\Delta^2 p\log(L L' L_\phi p/\varepsilon)}{\varepsilon^2} \] generators $G_{u_1}, ..., G_{u_T}$ such that, if $\mathcal{S}_u$ is the uniform distribution on $u_i$, and $D$ is a discriminator that outputs only 1/2, then $(\mathcal{S}_u, D)$ is an $\varepsilon$-approximate Nash equilibrium.
\end{aroralemma1}
\begin{proof}
Let $V$ be the value of the associated zero-sum game. One strategy of the discriminator is to just output 1/2. Since this strategy has payoff $2\phi(1/2)$ no matter what the generator does, it follows that $V\ge2\phi(1/2)$.

By assumption, for any point $x$ and any $\varepsilon>0$, there is a generator $G_{x,\varepsilon}$ such that $\ev{z}{\ndist}[\lVert G_{x,\varepsilon}(z)-x \rVert]\le\varepsilon$. For any $\zeta>0$, we can take a sample $x\sim\rdist$, and use the generator $G_{x,\zeta}$. Let $p_{\zeta}$ be the distribution generated by this mixture of generators. By the point mass approximation property of the generators, $\dive{W}{\rdist}{p_{\zeta}}\le\zeta$. Combining this with the fact that the discriminator is $L'$-Lipschitz and that $\phi$ is $L_\phi$-Lipschitz, it holds for any discriminator $D_v$ that 
\begin{equation}
\lvert \ev{x}{\rdist}[\phi(1-D(x))]-\ev{x}{p_{\zeta}}[\phi(1-D(x))]\rvert \le O(L_\phi L'\zeta).
\end{equation}
Therefore,
\begin{align*}
&\max_{v\in\mathcal{V}}\biggl( \ev{x}{\rdist}[\phi(D(x))]+\ev{x}{p_{\zeta}}[\phi(1-D(x))] \biggr) \\
&\le \max_{v\in\mathcal{V}}\biggl(\ev{x}{\rdist}[\phi(D(x))+\phi(1-D(x))]\biggr) + O(L_\phi L'\zeta) \\
&\le 2\phi(1/2) + O(L_\phi L'\zeta).
\end{align*}
Here, the last step uses the assumption that $\phi$ is concave. Hence $V\le2\phi(1/2)+O(L_\phi L'\zeta)$ for any $\zeta$. Letting $\zeta\to 0$, we get that $V = 2\phi(1/2)$.

Now let $(\mathcal{S}'_u, \mathcal{S}'_v)$ be the pair of optimal mixed strategies as given by Glicksberg's Theorem (Theorem~\ref{glickthm}), and let $V$ be the optimal value. We will show that, by randomly sampling $T$ generators from $\mathcal{S}'_u$, we get the desired mixture with high probability.

First, we construct $(\frac{\varepsilon}{4LL'L_\phi})$-nets for the discriminator parameter space $\mathcal{V}$. Let $A$ be the set of centres for such nets. Since $\mathcal{V}$ is compact, $A$ is finite. Moreover, by a standard construction of such nets, we have that \[ \log(\lvert A\rvert)\le C'n\log(LL'L_\phi \cdot p/\varepsilon) \] for some constant $C'$. Let $u_1, ..., u_T$ be independent samples from $\mathcal{S}'_u$. By the Chernoff bound, for any net centre $a\in A$,
\begin{equation}
\mathbb{P}\biggl(\mathbb{E}_{i\in[T]}[F(u_i,a)] \ge \mathbb{E}_{u\in\mathcal{U}}[F(u,a)]+\varepsilon/2 \biggr) \le \exp\left(-\frac{\varepsilon^2T}{2\Delta^2}\right).
\end{equation}

Therefore, when $T=\frac{C\Delta^2 p\log(L L' L_\phi p/\varepsilon)}{\varepsilon^2}$ and the constant $C$ is greater than $2C'$, then with high probability this inequality is true for all $a\in A$. For any $v\in\mathcal{V}$, let $a'$ be the closest point in the net, so by construction $\lVert v-a'\rVert\le \frac{\varepsilon}{4LL'L_\phi}$. Using this, one can derive straightforwardly that $F(u,v)$ is $(2LL'L_\phi)$-Lipschitz in both $u$ and $v$, so that
\begin{equation}
\mathbb{E}_{i\in[T]}[F(u_i, a')]\le\mathbb{E}_{i\in[T]}[F(u_i, v)]+\varepsilon/2.
\end{equation}
We then get for any $v'\in\mathcal{V}$ that
\begin{equation}
\mathbb{E}_{i\in[T]}[F(u_i, v')]\le2\phi(1/2)+\varepsilon.
\end{equation}
Therefore, this mixture of generators can win against any discriminator, and by a probabilistic argument, there must exist such generators. Since the discriminator given by constant 1/2 can achieve the value $V$ regardless of the generator value, we get an approximate equilibrium.
\end{proof}

Given the existence of the approximate equilibrium, the next step to prove the existence of an approximate \textit{pure} equilibrium is to construct a single generator that can approximately generate the mixture distribution of generators. To do so, we pass our noise input $h\sim\ndist$ through all the generators $G_{u_1},...,G_{u_T}$ and implement a `multi-way selector' to select a uniformly random output from $G_{u_i}(h)$, where $i\in[T]$.

First, we show how to compute a step function using a two-layer neural network.

\begin{aroralemma2}[\citet{arora2017generalization}, Lemma 3]\label{al2}
Let $q$ be a positive integer and $z_1<...<z_q$. For any $0<\delta<\min_{i\in[q-1]}(z_{i+1}-z_i)$, there is a two-layer neural network with single input $h\in\mathbb{R}$ that outputs $q+1$ numbers $x_1,...,x_{q+1}$ such that
\begin{enumerate}
\item $\sum_{i=1}^{q+1}x_i=1$ for all $h$;
\item when $h\in[z_{i-1}+\delta/2, z_i-\delta/2]$, we have that $x_i=1$ and all the other $x_j$ are $0$. When $h\le z_1-\delta/2$ only $x_1$ is 1, and when $h\ge z_q+\delta/2$ only $x_{q+1}=1$.
\end{enumerate}
\end{aroralemma2}
\begin{proof}
Using a two-layer neural network with ReLU activation functions, we can compute the function 
\begin{equation}
f_i(h) = \max\left\{\frac{h-z_i-\delta/2}{\delta}, 0\right\} - \max\left\{\frac{h-z_i+\delta/2}{\delta}, 0\right\}.
\end{equation}
This function evaluates to 0 for all $h<z_i-\delta/2$, and 1 for all $h\ge z_i+\delta/2$, and changes linearly in-between. Writing
\begin{align*}
x_1 &= 1-f_1(h) \\
x_{q+1} &= f_q(h) \\
x_q &= f_i(h)-f_{i-1}(h) \quad \forall i\in\{2,3,...,q\},
\end{align*}
we see that these functions satisfy (1)-(2).
\end{proof}

Using these step functions, we can design the multi-way selector.

\begin{aroralemma3}[\citet{arora2017generalization}, Lemma 4]\label{al3}
In the setting above, suppose the generator and discriminator are both $k$-layer neural networks (where $k\ge2$), and the last layer uses the ReLU activation function. Then there is a $(k+1)$-layer neural network with $O\left(\frac{\Delta^2p^2\log(LL'L_\phi\cdot p/\varepsilon)}{\varepsilon^2}\right)$ parameters that can generate a distribution within $\delta$ TV distance of the mixture of $G_{u_1}, ..., G_{u_T}$.
\end{aroralemma3}
\begin{proof}[Proof Idea]
Since we have implemented step functions from Lemma~\ref{al2}, we can pass the input through all the generators $G_{u_1},...,G_{u_T}$. At the last layer of each $G_{u_i}$, we add a large multiple of $-(1-x_i)$, where $x_i$ is the $i$-th outputof the network in Lemma~\ref{al2}. Then, if $x_i=0$, this will effectively `de-activate' the network by bringing it below the threshold of the ReLU function. However, if $x_i=1$, this will have no effect. By Lemma~\ref{al2}, we know that most of the time only one of the $x_i$'s will be 1, so that only one generator is selected.
\end{proof}
\begin{proof}
Suppose the input for the generator is $(h_0, h)\sim\mathcal{N}(0,1)\times\ndist$, where samples are drawn independently. We will pass the input $h$ through the generators and get outputs $G_{u_i}(h)$, and then use $h_0$ to select one of these outputs as the `true' output.

Let $z_1, ..., z_{T-1}$ be real numbers that divide the probability density of a Gaussian into $T$ equal parts. Choose $\delta' = \delta/(100T)$ in Lemma~\ref{al2} to get a 2-layer neural net that computes step functions $x_1,...,x_T$. Then the probability that $(x_1,...,x_T)$ has more than one non-zero entry is smaller than $\delta$. Now, for the output of $G_{u_i}(h)$, in each output ReLU gate, add a multiple of $-(1-x_i)$ that is larger than than the maximum possible output. Therefore, when $x_i=0$, the result before the ReLU will be negative and so the output will be `disabled', and when $x_i=1$ the output will be preserved. Call this modified network $\hat{G}_{u_i}$. Then $\hat{G}_{u_i}=G_{u_i}$ when $x_i=1$ and $\hat{G}_{u_i}=0$ when $x_i=0$. Now add a layer that outputs the sum of $\hat{G}_{u_i}$.

By construction, when $(x_1,...,x_T)$ has only one non-zero entry, the network correctly outputs the corresponding $G_{u_i}(x_i)$. The probability that this happens is at least $1-\delta$, and so the TV distance with the mixture is bounded by $\delta$.
\end{proof}

\begin{arorathmrestated}[Theorem~\ref{arorapureeq} restated]
In the setting above, there exists $(k+1)$-layer neural networks of generators $G$ and discriminator $D$ with $O\left(\frac{\Delta^2p^2\log(LL'L_\phi\cdot p/\varepsilon)}{\varepsilon^2}\right)$ parameters, such that there exists an $\varepsilon$-approximate pure equilibrium with value $2\phi(1/2)$ to the game induced by $d_{\mathcal{F}, \phi}(\rdist \| \gdist)$.
\end{arorathmrestated}
\begin{proof}
Let $T$ be large enough so that there exists an $\varepsilon/2$-approximate mixed Nash equilibrium. Let the new set of generators be constructed as in Lemma~\ref{al3}, with $\delta\le\varepsilon/(4\Delta)$ and $G_{u_1},...,G_{u_T}$ as the original set of generators. Let $D$ be the discriminator that always outputs 1/2, and $G$ be the `multi-way selector' generator constructed by the $T$ generators from the approximate mixed equilibrium. Let $F^*(G, D)$ denote the value for the new two-player game.\footnote{The game is new, since the space $\mathcal{F}$ of neural nets now includes neural nets of $(k+1)$ layers instead of just $k$ layers.} For any discriminator $D_v$, we have
\begin{align*}
F^*(G, D_v) &\le \mathbb{E}_{i\in[T]}[F(u_i, v)]+\lvert F^*(G, D)-\mathbb{E}_{i\in[T]}F(u_i, v)\rvert \\
&\le V+\varepsilon/2+2\Delta\frac{\varepsilon}{4\Delta} \\
&\le V+\varepsilon.
\end{align*}
Here, the bound for the first term comes from Lemma~\ref{al1}, and the fact that the expectation is smaller than the maximum of expected values. The bound for the second term comes from the fact that changing a $\delta$ amount of probability mass can change the payoff $F$ by at most $2\Delta\delta$ (recalling that $\phi$ is bounded in $[-\Delta,\Delta]$). Therefore, the generator will still fool all discriminators, and we therefore get the required pure equilibrium.
\end{proof}


\bibliographystyle{abbrvnat}
\bibliography{dissbib}

\end{document}